\documentclass[twoside,11pt]{article}

\usepackage{jmlr2e}

\ShortHeadings{Multi-Stage Multi-Task Feature Learning}{Gong, Ye and Zhang}
\firstpageno{1}

\usepackage[ruled,linesnumbered,titlenumbered]{algorithm2e}
\usepackage{graphicx}
\usepackage{amsmath,amssymb,url}
\usepackage{float}
\usepackage{bm}
\usepackage{multirow}

\def\sign{\mathrm{sign}}

\def\Pr{\mathrm{Pr}}

\newcommand{\Argmin}{\mathop{\arg\min}}

\newcommand{\EqRef}[1]{Eq.~(\ref{#1})}
\newcommand{\SecRef}[1]{Section~\ref{#1}}

\newcommand{\ThemRef}[1]{Theorem~\ref{#1}}
\newcommand{\LemmaRef}[1]{Lemma~\ref{#1}}
\newcommand{\AssumeRef}[1]{Assumption~\ref{#1}}

\newcommand{\FigRef}[1]{Figure~\ref{#1}}
\newcommand{\AlgRef}[1]{Algorithm~\ref{#1}}
\newcommand{\TabRef}[1]{Table~\ref{#1}}

\newtheorem{assumption}{Assumption}

\begin{document}

\title{Multi-Stage Multi-Task Feature Learning}

\author{\name $^{\dag}$Pinghua Gong \email gph08@mails.tsinghua.edu.cn \\
       \name $^{\ddag}$Jieping Ye \email jieping.ye@asu.edu \\
       \name $^{\dag}$Changshui Zhang \email zcs@mail.tsinghua.edu.cn \\
       \addr $^{\dag}$State Key Laboratory on Intelligent Technology and Systems\\
       Tsinghua National Laboratory for Information Science and Technology (TNList)\\
       Department of Automation, Tsinghua University, Beijing 100084, China\\
       \addr $^{\ddag}$Computer Science and Engineering, Center for Evolutionary Medicine and Informatics\\
       The Biodesign Institute, Arizona State University, Tempe, AZ 85287, USA}

\editor{}

\maketitle

\begin{abstract}%   <- trailing '%' for backward compatibility of .sty file
Multi-task sparse feature learning aims to improve the generalization performance by exploiting
the shared features among tasks.
It has been successfully applied to many applications including computer vision and biomedical informatics.
Most of the existing multi-task sparse feature learning
algorithms are formulated as a convex sparse regularization problem, which is usually suboptimal, due
to its looseness for approximating an $\ell_0$-type regularizer.
In this paper, we propose a non-convex formulation
for multi-task sparse feature learning based on a novel non-convex regularizer.
To solve the non-convex optimization problem, we propose a
Multi-Stage Multi-Task Feature Learning (MSMTFL) algorithm; we also provide intuitive interpretations,
detailed convergence and reproducibility analysis for the proposed algorithm. Moreover, we present a detailed
theoretical analysis showing that MSMTFL achieves a better parameter estimation error bound
than the convex formulation.
Empirical studies on both synthetic and real-world data sets demonstrate the effectiveness of MSMTFL in comparison with
the state of the art multi-task sparse feature learning algorithms.
\end{abstract}


\section{Introduction}\label{sec:introduction}
Multi-task learning (MTL) \citep{caruana1997multitask}
exploits the relationships among multiple related tasks to
improve the generalization performance. It has been successfully applied to many
applications such as speech classification \citep{parameswaran2010large}, handwritten
character recognition \citep{obozinski2006multi,quadrianto2010multitask} and
medical diagnosis \citep{bi2008improved}.
One common assumption in multi-task learning is that all tasks should share
some common structures including the prior or parameters of Bayesian
models \citep{schwaighofer2005learning,yu2005learning,zhang2006learning}, a similarity metric matrix
\citep{parameswaran2010large}, a classification weight vector \citep{evgeniou2004regularized},
a low rank subspace \citep{chen2010learning,negahban2011estimation} and a common set of shared features
\citep{argyriou2008convex,gong2012robust,kim2009tree,kolar2011union,lounici2009taking,liu2009multi,negahban2008joint,obozinski2006multi,yang2009heterogeneous,zhang2010probabilistic}.

Multi-task feature learning, which aims to learn a common set of shared features,
has received a lot of interests in machine learning recently, due to the
popularity of various sparse learning formulations and their successful applications in many problems.
In this paper, we focus on a specific multi-task feature learning setting, in which we
learn the features specific to each task as well as the common features shared among tasks.
Although many multi-task feature learning
algorithms have been proposed in the past, many of them require
the relevant features to be shared by all tasks. This is too restrictive in real-world
applications \citep{jalali2010dirty}. To overcome this limitation, \citet{jalali2010dirty} proposed
an $\ell_1+\ell_{1,\infty}$ regularized formulation, called dirty model,
to leverage the common features shared among tasks. The dirty model allows
a certain feature to be shared by some tasks but not all tasks. \citet{jalali2010dirty} also presented
a theoretical analysis under the
incoherence condition \citep{donoho2006stable,guillaume2011support} which is more
restrictive than RIP \citep{candes2005decoding,zhang2012multi}.
The $\ell_1+\ell_{1,\infty}$ regularizer is a convex relaxation
for the $\ell_0$-type one, in which a globally optimal solution can
be obtained. However, a convex regularizer is known to too loose to
approximate the $\ell_0$-type one and often achieves
suboptimal performance (either require restrictive conditions
or obtain a suboptimal error bound) \citep{zhang2012general,zhang2010analysis,zhang2012multi}.
To remedy the limitation, a non-convex regularizer can be used instead. However,
the non-convex formulation is usually difficult to solve
and a globally optimal solution can not be obtained in most practical problems.
Moreover, the solution of the non-convex formulation heavily depends on the specific optimization algorithms employed.
Even with the same optimization algorithm adopted, different initializations usually lead to different solutions.
Thus, it is often challenging to analyze the theoretical behavior of a non-convex formulation.

\textbf{Contributions}:
We propose a non-convex formulation, called capped-$\ell_1$,$\ell_1$ regularized model for multi-task feature learning.
The proposed model aims to simultaneously learn the features specific to each task as well as the common features shared among tasks.
We propose a Multi-Stage Multi-Task Feature Learning (MSMTFL) algorithm to solve the non-convex optimization problem.
We also provide intuitive interpretations of the proposed algorithm from
several aspects. In addition, we present a detailed convergence analysis for the proposed algorithm.
To address the reproducibility issue of the non-convex formulation, we show that the solution generated by the MSMTFL algorithm
is unique (i.e., the solution is reproducible) under a mild condition, which facilitates the theoretical analysis
of the MSMTFL algorithm. Although the MSMTFL algorithm may not obtain a globally optimal solution,
we show that this solution achieves good performance. Specifically, we present a detailed theoretical analysis on
the parameter estimation error bound for the MSMTFL algorithm.
Our analysis shows that, under the sparse eigenvalue condition which is \emph{weaker} than the incoherence condition used in \citet{jalali2010dirty},
MSMTFL improves the error bound during the multi-stage iteration, i.e., the error bound at the current iteration improves the one at the last iteration. Empirical studies on both synthetic and real-world data sets demonstrate the effectiveness of the
MSMTFL algorithm in comparison with the state of the art algorithms.

\textbf{Notations}:
Scalars and vectors are denoted by lower case letters and
bold face lower case letters, respectively. Matrices and sets are denoted by capital
letters and calligraphic capital letters, respectively. The $\ell_1$ norm, Euclidean norm,
$\ell_\infty$ norm and Frobenius norm are denoted
by $\|\cdot\|_1,~\|\cdot\|$, $\|\cdot\|_{\infty}$ and $\|\cdot\|_F$, respectively.
$|\cdot|$ denotes the absolute value of a scalar or the number of
elements in a set, depending on the context. We define the $\ell_{p,q}$ norm of
a matrix $X$ as $\|X\|_{p,q}=\left(\sum_{i}\left((\sum_{j}|x_{ij}|^q)^{1/q}\right)^p\right)^{1/p}$.
We define $\mathbb{N}_n$ as $\{1,\cdots,n\}$ and $N(\mu,\sigma^2)$
as the normal distribution with mean $\mu$ and variance $\sigma^2$.
For a $d\times m$ matrix $W$ and sets $\mathcal{I}_i\subseteq\mathbb{N}_d\times\{i\},\mathcal{I}\subseteq\mathbb{N}_d\times\mathbb{N}_d$,
we let $\mathbf{w}_{\mathcal{I}_i}$ be the $d\times 1$ vector with the $j$-th entry being $w_{ji}$, if $(j,i)\in\mathcal{I}_i$, and $0$, otherwise.
We also let $W_{\mathcal{I}}$ be a $d\times m$ matrix with the $(j,i)$-th entry being $w_{ji}$, if $(j,i)\in\mathcal{I}$, and $0$, otherwise.

\textbf{Organization}: In \SecRef{sec:formulation}, we introduce a non-convex formulation and present the corresponding
optimization algorithm. In \SecRef{sec:convandreprod}, we discuss the convergence and reproducibility issues of the MSMTFL algorithm.
In \SecRef{sec:theoreticalanalysis}, we present a detailed theoretical
analysis on the MSMTFL algorithm, in terms of the parameter estimation error bound.
In \SecRef{sec:proof}, we provide a sketch of the proof of the presented theoretical results and
the detailed proof is provided in the Appendix. In \SecRef{sec:experiments}, we report
the experimental results and we conclude the paper in \SecRef{sec:conclusions}.

%This paper is substantially extended from a short version ''Multi-Stage Multi-Task Feature Learning'' which will appear in NIPS 2012 \citep{gong2012multi}.
%The main extensions include: (1) We present two intuitive interpretations (locally linear approximation and block coordinate descent) for the MSMTFL algorithm, which makes the proposed algorithm easy-understood. (2) We present a detailed convergence analysis for the MSMTFL algorithm, based on the block coordinate descent interpretation. (3) We address the reproducibility issue of the proposed non-convex formulation, that is, we discuss the uniqueness of the solution of the proposed non-convex formulation. (4) More discussions for the MSMTFL algorithm is presented.

\section{The Proposed Formulation and the Optimization Algorithm}\label{sec:formulation}
In this section, we first present a non-convex formulation for multi-task feature learning.
Then, we show how to solve the corresponding optimization problem. Finally, we provide intuitive
interpretations and discussions for the proposed algorithm.

\subsection{A Non-convex Formulation}
Assume we are given $m$ learning tasks associated with training data
$\left\{(X_1,\mathbf{y}_1),\cdots,(X_m,\mathbf{y}_m)\right\}$,
where $X_i\in\mathbb{R}^{n_i\times d}$ is the data matrix of the $i$-th task with
each row as a sample; $\mathbf{y}_i\in\mathbb{R}^{n_i}$ is the response of the $i$-th task;
$d$ is the data dimensionality; $n_i$ is the number of samples
for the $i$-th task. We consider learning a weight matrix
$W=[\mathbf{w}_1,\cdots,\mathbf{w}_m]\in\mathbb{R}^{d\times m}$ consisting of the weight vectors for $m$ linear predictive models:
$\mathbf{y}_i\approx\mathbf{f}_i(X_i)=X_i\mathbf{w}_i,~i\in\mathbb{N}_m$.
%\begin{align}
%\mathbf{y}_i\approx\mathbf{f}_i(X_i)=X_i\mathbf{w}_i,~i\in\mathbb{N}_m.\nonumber
%\end{align}
In this paper, we propose a non-convex multi-task feature learning formulation to learn these $m$ models simultaneously, based on the capped-$\ell_1$,$\ell_1$ regularization. Specifically, we first impose the $\ell_1$ penalty on each row of $W$, obtaining
a column vector. Then, we impose the capped-$\ell_1$ penalty \citep{zhang2010analysis,zhang2012multi} on that vector. Formally, we formulate
our proposed model as follows:
\begin{align}\label{eq:msmtfl}
\min_{W\in\mathbb{R}^{d\times m}}\left\{l(W)+\lambda\sum_{j=1}^d\min\left(\|\mathbf{w}^j\|_1,\theta\right)\right\},
\end{align}
where $l(W)$ is an empirical loss function of $W$; $\lambda~(>0)$ is a
parameter balancing the empirical loss and the regularization;
$\theta~(>0)$ is a thresholding parameter; $\mathbf{w}^j$ is the
$j$-th row of the matrix $W$. In this paper, we focus on the following quadratic loss function:
\begin{align}\label{eq:lossfunction}
l(W)=\sum_{i=1}^m\frac{1}{mn_i}\left\|X_i\mathbf{w}_i-\mathbf{y}_{i}\right\|^2.
\end{align}

Intuitively, due to the capped-$\ell_1,\ell_1$ penalty, the optimal solution of \EqRef{eq:msmtfl}
denoted as $W^\star$ has many zero rows.
For a nonzero row $(\mathbf{w}^\star)^k$, some entries
may be zero, due to the $\ell_1$-norm imposed on each row of $W$. Thus, under the formulation in \EqRef{eq:msmtfl},
some features can be shared by some tasks but not all the tasks. Therefore,
the proposed formulation can leverage the common features shared among tasks.

\subsection{Optimization Algorithm}
The formulation in \EqRef{eq:msmtfl} is non-convex and is difficult to solve.
In this paper, we propose an algorithm called Multi-Stage Multi-Task Feature Learning (MSMTFL) to solve the
optimization problem (see details in \AlgRef{alg:msmtfl}). In this algorithm, a key step is how to
efficiently solve \EqRef{eq:relaxedlasso}. Observing that the objective function in \EqRef{eq:relaxedlasso} can
be decomposed into the sum of a differential loss function and a non-differential regularization term, we employ FISTA \citep{beck2009fast} to solve the
sub-problem. In the following, we present some intuitive interpretations of
the proposed algorithm from several aspects.
\vspace{-0.0cm}
\begin{algorithm}[!ht]\label{alg:msmtfl}
Initialize $\lambda_j^{(0)}=\lambda$\;
   \For{$\ell=1,2,\cdots$}
   {
   Let $\hat{W}^{(\ell)}$ be a solution of the following problem:
   \begin{align}\label{eq:relaxedlasso}
   \min_{W\in\mathbb{R}^{d\times m}}\left\{l(W)+\sum_{j=1}^d\lambda_j^{(\ell-1)}\|\mathbf{w}^j\|_1\right\}.
   \end{align}

   Let $\lambda_j^{(\ell)}=\lambda I(\|(\hat{\mathbf{w}}^{(\ell)})^j\|_1<\theta)~(j=1,\cdots,d)$, where
   $(\hat{\mathbf{w}}^{(\ell)})^j$ is the $j$-th row of $\hat{W}^{(\ell)}$ and
   $I(\cdot)$ denotes the $\{0,1\}$-valued indicator function.
   }
\caption{MSMTFL: Multi-Stage Multi-Task Feature Learning}
\end{algorithm}\vspace{-0.0cm}

\subsubsection{Locally Linear Approximation}
First, we define two auxiliary functions:
\begin{align}
&\mathbf{h}:\mathbb{R}^{d\times m}\mapsto\mathbb{R}^{d}_+,~\mathbf{h}(W) = \left[\|\mathbf{w}^1\|_1,\cdots,\|\mathbf{w}^d\|_1\right]^T,\nonumber\\
&g:\mathbb{R}^{d}_+\mapsto\mathbb{R}_+,~g(\mathbf{u})=\sum_{j=1}^d\min(u_j,\theta).\nonumber
\end{align}
We note that $g(\cdot)$ is a concave function and we say that a vector $\mathbf{s}\in\mathbb{R}^d$ is a sub-gradient of $g$ at $\mathbf{v}\in\mathbb{R}^d_+$, if for all vector $\mathbf{u}\in\mathbb{R}^d_+$, the following inequality holds:
\begin{align}
g(\mathbf{u})\leq g(\mathbf{v})+\langle\mathbf{s}, \mathbf{u}-\mathbf{v}\rangle,\nonumber
\end{align}
where $\langle\cdot\rangle$ denotes the inner product.
Using the functions defined above, \EqRef{eq:msmtfl} can be equivalently rewritten as follows:
\begin{align}\label{eq:convexconcavemsmtfl}
\min_{W\in\mathbb{R}^{d\times m}} \left\{l(W) + \lambda g(\mathbf{h}(W))\right\}.
\end{align}
Based on the definition of the sub-gradient for a concave function given above, we can obtain an upper bound of $g(\mathbf{h}(W))$ using a locally linear approximation at $\mathbf{h}(\hat{W}^{(\ell)})$:
\begin{align}
g(\mathbf{h}(W))\leq g(\mathbf{h}(\hat{W}^{(\ell)})) + \left\langle\mathbf{s}^{(\ell)}, \mathbf{h}(W)-\mathbf{h}(\hat{W}^{(\ell)})\right\rangle,\nonumber
\end{align}
where $\mathbf{s}^{(\ell)}$ is a sub-gradient of $g(\mathbf{u})$ at $\mathbf{u}=\mathbf{h}(\hat{W}^{(\ell)})$.
Furthermore, we can obtain an upper bound of the objective function in \EqRef{eq:convexconcavemsmtfl}, if the solution $\hat{W}^{(\ell)}$ at the $\ell$-th iteration is available:
\begin{align}
\forall W\in\mathbb{R}^{d\times m}:l(W)+\lambda g(\mathbf{h}(W))\leq l(W) + \lambda g(\mathbf{h}(\hat{W}^{(\ell)})) + \lambda\left\langle\mathbf{s}^{(\ell)}, \mathbf{h}(W)-\mathbf{h}(\hat{W}^{(\ell)})\right\rangle.\label{eq:upperbound}
\end{align}
It can be shown that a sub-gradient of $g(\mathbf{u})$ at $\mathbf{u}=\mathbf{h}(\hat{W}^{(\ell)})$ is
\begin{align}
\mathbf{s}^{(\ell)}=\left[I(\|(\hat{\mathbf{w}}^{(\ell)})^1\|_1<\theta),\cdots,I(\|(\hat{\mathbf{w}}^{(\ell)})^d\|_1<\theta)\right]^T,\label{eq:subgradient}
\end{align}
which is used in Step 4 of \AlgRef{alg:msmtfl}.
Since both $\lambda$ and $\mathbf{h}(\hat{W}^{(\ell)})$ are constant with respect to $W$, we have
\begin{align}
\hat{W}^{(\ell+1)}&=\Argmin_W \left\{l(W) + \lambda g(\mathbf{h}(\hat{W}^{(\ell)})) + \lambda\left\langle\mathbf{s}^{(\ell)}, \mathbf{h}(W)-\mathbf{h}(\hat{W}^{(\ell)})\right\rangle\right\}\nonumber\\
&=\Argmin_W \left\{l(W) + \lambda(\mathbf{s}^{(\ell)})^T \mathbf{h}(W)\right\},\nonumber
\end{align}
which, as shown in Step 3 of \AlgRef{alg:msmtfl}, obtains
the next iterative solution by minimizing the upper bound of the objective function in \EqRef{eq:convexconcavemsmtfl}. Thus, in the viewpoint of
the locally linear approximation, we can understand \AlgRef{alg:msmtfl} as follows: The original formulation in \EqRef{eq:convexconcavemsmtfl} is non-convex and is difficult to solve; the proposed algorithm minimizes an upper bound in each step, which is convex and can be solved efficiently.
%Actually, this idea is very common in some other well-known algorithms such as Expectation Minimization (EM) algorithm \citep{dempster1977maximum},
%the multiplicative update algorithm in Nonnegative Matrix Factorization (NMF) \citep{lee1999learning,lee2001algorithms} and
It is closely related to the Concave Convex Procedure (CCCP) \citep{yuille2003concave}.
In addition, we can easily verify that the objective function value decreases monotonically as follows:
\begin{align}
l(\hat{W}^{(\ell+1)})+\lambda g(\mathbf{h}(\hat{W}^{(\ell+1)}))
&\leq l(\hat{W}^{(\ell+1)})+ \lambda g(\mathbf{h}(\hat{W}^{(\ell)})) + \lambda\left\langle\mathbf{s}^{(\ell)}, \mathbf{h}(\hat{W}^{(\ell+1)})-\mathbf{h}(\hat{W}^{(\ell)})\right\rangle\nonumber\\
&\leq l(\hat{W}^{(\ell)})+ \lambda g(\mathbf{h}(\hat{W}^{(\ell)})) + \lambda\left\langle\mathbf{s}^{(\ell)}, \mathbf{h}(\hat{W}^{(\ell)})-\mathbf{h}(\hat{W}^{(\ell)})\right\rangle\nonumber\\
&=l(\hat{W}^{(\ell)})+ \lambda g(\mathbf{h}(\hat{W}^{(\ell)})),\nonumber
\end{align}
where the first inequality is due to \EqRef{eq:upperbound} and the second inequality follows from the fact that $\hat{W}^{(\ell+1)}$
is a minimizer of the right hand side of \EqRef{eq:upperbound}.

An important issue we should mention is that a monotonic decrease of the objective function value does not guarantee
the convergence of the algorithm, even if the objective function
is strictly convex and continuously differentiable (see an example in the book \citep[Fig 1.2.6]{bertsekas1999nonlinear}).
In \SecRef{sec:convergence}, we will formally discuss the convergence issue.

\subsubsection{Block Coordinate Descent}
Recall that $g(\mathbf{u})$ is a concave function. We can define its conjugate function as \citep{rockafellar1970convex}:
\begin{align}
g^\star(\mathbf{v})=\inf_{\mathbf{u}}\{\mathbf{v}^T\mathbf{u}-g(\mathbf{u})\}.\nonumber
\end{align}
Since $g(\mathbf{u})$ is also a closed function (i.e., the epigraph of $g(\mathbf{u})$ is convex),
the conjugate function of $g^\star(\mathbf{v})$ is the original function $g(\mathbf{u})$ \citep[Chap. 5.4]{bertsekas1999nonlinear}, that is:
\begin{align}
g(\mathbf{u})=\inf_{\mathbf{v}}\{\mathbf{u}^T\mathbf{v}-g^\star(\mathbf{v})\}.\label{eq:conjofconj}
\end{align}
Substituting \EqRef{eq:conjofconj} with $\mathbf{u}=\mathbf{h}(W)$ into \EqRef{eq:convexconcavemsmtfl}, we can reformulate \EqRef{eq:convexconcavemsmtfl} as:
\begin{align}
\min_{W,\mathbf{v}}\left\{f(W,\mathbf{v})=l(W) + \lambda \mathbf{v}^T\mathbf{h}(W)-\lambda g^\star(\mathbf{v})\right\}\label{eq:bcdformulation}
\end{align}
A straightforward algorithm for optimizing \EqRef{eq:bcdformulation} is the block coordinate descent \citep{grippo2000convergence,tseng2001convergence} summarized below:
\begin{itemize}
\item
Fix $W=\hat{W}^{(\ell)}$:
\begin{align}
\hat{\mathbf{v}}^{(\ell)}&=\Argmin_{\mathbf{v}}\left\{l(\hat{W}^{(\ell)})+\lambda\mathbf{v}^T\mathbf{h}(\hat{W}^{(\ell)})-\lambda g^\star(\mathbf{v})\right\}\nonumber\\
&=\Argmin_{\mathbf{v}}\left\{\mathbf{v}^T\mathbf{h}(\hat{W}^{(\ell)})-g^\star(\mathbf{v})\right\}.\label{eq:bcdvopt}
\end{align}
Based on \EqRef{eq:conjofconj} and the Danskin's Theorem \citep[Proposition B.25]{bertsekas1999nonlinear}, one solution of \EqRef{eq:bcdvopt} is given by a sub-gradient of
$g(\mathbf{u})$ at $\mathbf{u}=\mathbf{h}(\hat{W}^{(\ell)})$. That is, we can choose $\hat{\mathbf{v}}^{(\ell)}=\mathbf{s}^{(\ell)}$ given in \EqRef{eq:subgradient}.
Apparently, \EqRef{eq:bcdvopt} is equivalent to Step 4 in \AlgRef{alg:msmtfl}.
\item
Fix $\mathbf{v}=\hat{\mathbf{v}}^{(\ell)}=\left[I(\|(\hat{\mathbf{w}}^{(\ell)})^1\|_1<\theta),\cdots,I(\|(\hat{\mathbf{w}}^{(\ell)})^d\|_1<\theta)\right]^T$:
\begin{align}
\hat{W}^{(\ell+1)}&=\Argmin_{W}\left\{l(W)+\lambda(\hat{\mathbf{v}}^{(\ell)})^T\mathbf{h}(W)-\lambda g^\star(\hat{\mathbf{v}}^{(\ell)})\right\}\nonumber\\
&=\Argmin_{W}\left\{l(W) + \lambda (\hat{\mathbf{v}}^{(\ell)})^T\mathbf{h}(W)\right\},\label{eq:bcdWopt}
\end{align}
which corresponds to Step 3 of \AlgRef{alg:msmtfl}.
\end{itemize}
The block coordinate descent procedure is intuitive, however, it is non-trivial to analyze its convergence behavior. We will present the convergence analysis in \SecRef{sec:convergence}.

\subsubsection{Discussions}
If we terminate the algorithm
with $\ell=1$, the MSMTFL algorithm is equivalent to the $\ell_1$ regularized multi-task feature learning algorithm (Lasso).
Thus, the solution obtained by MSMTFL can be considered as a multi-stage refinement of that of Lasso. Basically, the MSMTFL algorithm
solves a sequence of weighted Lasso problems, where the weights $\lambda_j$'s are set as the product of the parameter $\lambda$ in \EqRef{eq:msmtfl}
and a $\{0,1\}$-valued indicator function.
Specifically, a penalty is imposed in the current stage if the $\ell_1$-norm of some row of $W$ in the last stage is smaller than the threshold $\theta$; otherwise, no penalty is imposed.
In other words, MSMTFL in the current stage tends to shrink the small rows of $W$ and keep the large rows of $W$ in the last stage. However, Lasso (corresponds to $\ell=1$) penalizes all rows
of $W$ in the same way. It may incorrectly keep the irrelevant rows (which should have been zero rows) or shrink the relevant rows
(which should have been large rows) to be zero vectors. MSMTFL overcomes this limitation by adaptively penalizing the rows of $W$ according
to the solution generated in the last stage.
%It might have the two types of incorrectly generated rows of $W$ in the last stage corrected,
%keeping consistent with the ground truth in the subsequent stages.
One important question is whether the MSMTFL algorithm can improve the performance during the multi-stage iteration. In \SecRef{sec:theoreticalanalysis},
we will theoretically show that the MSMTFL algorithm indeed achieves the stagewise improvement in terms of the parameter estimation error bound. That is,
the error bound in the current stage improves the one in the last stage. Empirical studies in \SecRef{sec:experiments} also validate the presented theoretical analysis.

\section{Convergence and Reproducibility Analysis}\label{sec:convandreprod}
In this section, we first present the convergence analysis. Then, we discuss the reproducibility issue for the MSMTFL algorithm.

\subsection{Convergence Analysis}\label{sec:convergence}
The main convergence result is summarized in the following theorem, which is based on the block coordinate descent interpretation.
\begin{theorem}\label{theorem:convergence}
Let $(W^{\star},\mathbf{v}^{\star})$ be a limit point of the sequence $\{\hat{W}^{(\ell)},\hat{\mathbf{v}}^{(\ell)}\}$ generated by the block coordinate descent algorithm. Then $W^{\star}$ is a critical point of \EqRef{eq:msmtfl}.
\end{theorem}
\begin{proof}
Based on \EqRef{eq:bcdvopt} and \EqRef{eq:bcdWopt}, we have
\begin{align}
f(\hat{W}^{(\ell)},\hat{\mathbf{v}}^{(\ell)})&\leq f(\hat{W}^{(\ell)},\mathbf{v}),~\forall\mathbf{v}\in\mathbb{R}^d,\nonumber\\
f(\hat{W}^{(\ell+1)},\hat{\mathbf{v}}^{(\ell)})&\leq f(W,\hat{\mathbf{v}}^{(\ell)}),~\forall W\in\mathbb{R}^{d\times m}.\label{eq:decreasing}
\end{align}
It follows that
\begin{align}
f(\hat{W}^{(\ell+1)},\hat{\mathbf{v}}^{(\ell+1)})\leq f(\hat{W}^{(\ell+1)},\hat{\mathbf{v}}^{(\ell)})\leq f(\hat{W}^{(\ell)},\hat{\mathbf{v}}^{(\ell)}),\nonumber
\end{align}
which indicates that the sequence $\{f(\hat{W}^{(\ell)},\hat{\mathbf{v}}^{(\ell)})\}$ is monotonically decreasing. Since $(W^{\star},\mathbf{v}^{\star})$ is a limit point of $\{\hat{W}^{(\ell)},\hat{\mathbf{v}}^{(\ell)}\}$, there exists a subsequence $\mathcal{K}$ such that
\begin{align}
\lim_{\ell\in\mathcal{K}\rightarrow\infty}(\hat{W}^{(\ell)},\hat{\mathbf{v}}^{(\ell)})=(W^{\star},\mathbf{v}^{\star}).\nonumber
\end{align}
We observe that
\begin{align}
f(W,\mathbf{v})&=l(W) + \lambda \mathbf{v}^T\mathbf{h}(W)-\lambda g^\star(\mathbf{v})\nonumber\\
&\geq l(W) + \lambda g(\mathbf{h}(W))\geq 0,\nonumber
\end{align}
where the first inequality above is due to \EqRef{eq:conjofconj}. Thus,
$\{f(\hat{W}^{(\ell)},\hat{\mathbf{v}}^{(\ell)})\}_{\ell\in\mathcal{K}}$ is bounded below. Together with the fact that
$\{f(\hat{W}^{(\ell)},\hat{\mathbf{v}}^{(\ell)})\}$ is decreasing,
$\lim_{\ell\rightarrow\infty}f(\hat{W}^{(\ell)},\hat{\mathbf{v}}^{(\ell)})> -\infty$ exists. Since $f(W,\mathbf{v})$ is continuous, we have
\begin{align}
\lim_{\ell\rightarrow\infty}f(\hat{W}^{(\ell)},\hat{\mathbf{v}}^{(\ell)})=\lim_{\ell\in\mathcal{K}\rightarrow\infty}f(\hat{W}^{(\ell)},\hat{\mathbf{v}}^{(\ell)})=f(W^{\star},\mathbf{v}^{\star}).\nonumber
\end{align}
Taking limits on both sides of \EqRef{eq:decreasing} with $\ell\in\mathcal{K}\rightarrow\infty$, we have
\begin{align}
f(W^{\star},\mathbf{v}^{\star})\leq f(W,\mathbf{v}^{\star}),~\forall W\in\mathbb{R}^{d\times m},\nonumber
\end{align}
which implies
\begin{align}
W^{\star} &\in \Argmin_{W}f(W,\mathbf{v}^{\star})\nonumber\\
&=\Argmin_{W}\left\{l(W) + \lambda (\mathbf{v}^{\star})^T\mathbf{h}(W)-\lambda g^\star(\mathbf{v}^{\star})\right\}\nonumber\\
&=\Argmin_{W}\left\{l(W) + \lambda (\mathbf{v}^{\star})^T\mathbf{h}(W)\right\}.\label{eq:mincondition}
\end{align}
Therefore, the zero matrix $O$ must be a sub-gradient of the objective function in \EqRef{eq:mincondition} at $W=W^\star$ :
\begin{align}
O\in\partial l(W^\star)+\lambda\partial\left((\mathbf{v}^{\star})^T\mathbf{h}(W^\star)\right)=\partial l(W^\star)+\lambda\sum_{j=1}^dv^\star_j\partial\left(\|(\mathbf{w}^\star)^j\|_1\right),\label{eq:optcondition}
\end{align}
where $\partial l(W^\star)$ denotes the sub-differential (which is a set composed of all sub-gradients) of $l(W)$ at $W=W^\star$.
We observe that
\begin{align}
\hat{\mathbf{v}}^{(\ell)}\in\partial g(\mathbf{u})|_{\mathbf{u}=\mathbf{h}(\hat{W}^{(\ell)})},\nonumber
\end{align}
which implies that $\forall \mathbf{x}\in\mathbb{R}^d_+$:
\begin{align}
g(\mathbf{x})\leq g(\mathbf{h}(\hat{W}^{(\ell)}))+\left\langle\hat{\mathbf{v}}^{(\ell)},\mathbf{x}-\mathbf{h}(\hat{W}^{(\ell)})\right\rangle.\nonumber
\end{align}
Taking limits on both sides of the above inequality with $\ell\in\mathcal{K}\rightarrow\infty$, we have:
\begin{align}
g(\mathbf{x})\leq g(\mathbf{h}(W^{\star}))+\left\langle\mathbf{v}^{\star},\mathbf{x}-\mathbf{h}(W^{\star})\right\rangle,\nonumber
\end{align}
which implies that $\mathbf{v}^{\star}$ is a sub-gradient of $g(\mathbf{u})$ at $\mathbf{u}=\mathbf{h}(W^\star)$, that is:
\begin{align}
\mathbf{v}^\star\in\partial g(\mathbf{u})|_{\mathbf{u}=\mathbf{h}(W^\star)}.\label{eq:limitptsubgrad}
\end{align}
Substituting \EqRef{eq:limitptsubgrad} into \EqRef{eq:optcondition}, we obtain:
\begin{align}
O\in\partial l(W^\star)+\lambda\sum_{j=1}^d\partial\min(\|(\mathbf{w}^\star)^j\|_1,\theta).\nonumber
\end{align}
Therefore, $W^{\star}$ is a critical point of \EqRef{eq:msmtfl}. This completes the proof of \ThemRef{theorem:convergence}.
\end{proof}
\begin{remark}
Note that the above theorem holds by assuming that there exists a limit point. Next, we need to
prove that the sequence $\{\hat{W}^{(\ell)},\hat{\mathbf{v}}^{(\ell)}\}$ has a limit point. For any bounded initial point
$\{\hat{W}^{(0)},\hat{\mathbf{v}}^{(0)}\}$, based on \EqRef{eq:conjofconj}, \EqRef{eq:bcdformulation} and the monotonicity of $f(\hat{W}^{(\ell)},\hat{\mathbf{v}}^{(\ell)})$, we have:
\begin{align}
l(\hat{W}^{(\ell)})+\lambda g(\mathbf{h}(\hat{W}^{(\ell)}))\leq f(\hat{W}^{(\ell)},\hat{\mathbf{v}}^{(\ell)})\leq f(\hat{W}^{(0)},\hat{\mathbf{v}}^{(0)})<+\infty,\forall \ell\geq 1.\label{eq:boundedsequence}
\end{align}
Assume that the sequence $\{\hat{W}^{(\ell)},\hat{\mathbf{v}}^{(\ell)}\}$ is unbounded, that is, there exist some $i,j$ such that $|\hat{W}_{ij}^{(\ell)}|\rightarrow+\infty$.
It implies that $l(\hat{W}^{(\ell)})+\lambda g(\mathbf{h}(\hat{W}^{(\ell)}))\rightarrow+\infty$ (We exclude the case that
some columns of $X_i$ are zero vectors. Otherwise, we can simply remove the corresponding zero columns.) and hence $f(\hat{W}^{(\ell)},\hat{\mathbf{v}}^{(\ell)})\rightarrow+\infty$.
This leads to a contradiction with \EqRef{eq:boundedsequence}. Thus, the sequence $\{\hat{W}^{(\ell)},\hat{\mathbf{v}}^{(\ell)}\}$ is bounded and there exists at least one limit point $(W^{\star},\mathbf{v}^{\star})$, since any bounded sequence has limit points.
\end{remark}

Due to the equivalence between \AlgRef{alg:msmtfl} and the block coordinate descent algorithm above,
\ThemRef{theorem:convergence} and its remark indicate that the sequence $\{\hat{W}^{(\ell)}\}$ generated by \AlgRef{alg:msmtfl}
has at least one limit point that is also a critical point of \EqRef{eq:msmtfl}. The remaining issue is to analyze the performance of the critical point.
In the sequel, we will conduct analysis in two aspects: reproducibility and the parameter estimation performance.

\subsection{Reproducibility of The Algorithm}\label{sec:reproducibility}
In general, it is difficult to analyze the performance of a non-convex formulation, as different
solutions can be obtained due to different initializations. One natural question is whether the solution
generated by \AlgRef{alg:msmtfl} (based on the initialization of $\lambda_j^{(0)}=\lambda$ in Step 1) is reproducible. In other words, is the solution of \AlgRef{alg:msmtfl} unique?
If we can guarantee that, for any $\ell\geq 1$, the solution $\hat{W}^{(\ell)}$ of \EqRef{eq:relaxedlasso}
is unique, then the solution generated by \AlgRef{alg:msmtfl} is unique. That is,
the solution is reproducible. The main result is summarized in the following theorem:
\begin{theorem}\label{theorem:unique}
If $X_i\in\mathbb{R}^{n_i\times d}~(i\in\mathbb{N}_m)$ has entries drawn from a continuous probability distribution on $R^{n_id}$,
then, for any $\ell\geq 1$, the optimization problem in \EqRef{eq:relaxedlasso} has a unique solution with probability one.
\end{theorem}
\begin{proof}
\EqRef{eq:relaxedlasso} can be decomposed into $m$ independent smaller minimization problems:
\begin{align}
\hat{\mathbf{w}}^{(\ell)}_i=\Argmin_{\mathbf{w}_i\in\mathbb{R}^d}\frac{1}{mn_i}\|X_i\mathbf{w}_i-\mathbf{y}_i\|^2+\sum_{j=1}^d\lambda^{(\ell-1)}_j|w_{ji}|.\nonumber
\end{align}
Next, we only need to prove the solution of the above optimization problem is unique. To simplify the notations, we unclutter the above equation (by ignoring some superscripts
and subscripts) as follows:
\begin{align}
\hat{\mathbf{w}}=\Argmin_{\mathbf{w}\in\mathbb{R}^d}\frac{1}{mn}\|X\mathbf{w}-\mathbf{y}\|^2+\sum_{j=1}^d\lambda_j|w_{j}|,\label{eq:smallerprob}
\end{align}
The first order optimal condition is $\forall j\in\mathbb{N}_d$:
\begin{align}
\frac{2}{mn}\mathbf{x}_j^T(\mathbf{y}-X\hat{\mathbf{w}})=\lambda_j\sign(\hat{w}_j),\label{eq:firstorderopt}
\end{align}
where $\sign(\hat{w}_j)=1$, if $\hat{w}_j>0$; $\sign(\hat{w}_j)=-1$, if $\hat{w}_j<0$;
and $\sign(\hat{w}_j)\in[-1,1]$, otherwise. We define
\begin{align}
\mathcal{E}&=\left\{j\in\mathbb{N}_d: \frac{2}{mn}|\mathbf{x}_j^T(\mathbf{y}-X\hat{\mathbf{w}})|=\lambda_j\right\},\nonumber\\
&\mathbf{s}=\sign\left(\frac{2}{mn}X_{\mathcal{E}}^T(\mathbf{y}-X\hat{\mathbf{w}})\right),\nonumber
\end{align}
where $X_{\mathcal{E}}$ denotes the matrix composed of the columns of $X$ indexed by $\mathcal{E}$.
Then, the optimal solution $\hat{\mathbf{w}}$ of \EqRef{eq:smallerprob} satisfies
\begin{align}
&\hat{\mathbf{w}}_{\mathbb{N}_d\setminus\mathcal{E}}=\mathbf{0},\nonumber\\
&\hat{\mathbf{w}}_{\mathcal{E}}=\Argmin_{\mathbf{w}_{\mathcal{E}}\in\mathbb{R}^{|\mathcal{E}|}}\frac{1}{mn}\|X_{\mathcal{E}}\mathbf{w}_{\mathcal{E}}-\mathbf{y}\|^2+\sum_{j\in\mathcal{E}}\lambda_j|w_{j}|,~s.t.~s_jw_j\geq 0,j\in\mathcal{E},\label{eq:reducedprob}
\end{align}
where $\mathbf{w}_{\mathcal{E}}$ denotes the vector composed of entries of $\mathbf{w}$ indexed by $\mathcal{E}$.
Since $X\in\mathbb{R}^{n_i\times d}$ is drawn from the continuous probability distribution,
$X$ has columns in general positions with probability one and hence
$\mathrm{rank}(X_{\mathcal{E}})=|\mathcal{E}|$ (or equivalently $\mathrm{Null}(X_{\mathcal{E}})=\{\mathbf{0}\}$), due to
Lemma 3, Lemma 4 and their discussions in \citet{tibshirani2012lasso}. Therefore, the objective function in \EqRef{eq:reducedprob} is strictly convex, which implies that $\hat{\mathbf{w}}_{\mathcal{E}}$ is unique. Thus, the optimal solution $\hat{\mathbf{w}}$ of \EqRef{eq:smallerprob} is also unique and so is the optimization problem in \EqRef{eq:relaxedlasso} for any $\ell\geq 1$. This completes the proof of \ThemRef{theorem:unique}.
\end{proof}
\ThemRef{theorem:unique} is important in the sense that it makes the theoretical analysis for the parameter estimation performance of \AlgRef{alg:msmtfl} possible.
%Only under the very mild condition that all the data matrices $X_i$'s are sampled from the continuous probability distribution, the solution generated by \AlgRef{alg:msmtfl}
%is unique with probability one.
Although the solution may not be globally optimal, we show in the next section that the solution has good performance in terms of the parameter estimation error bound.

\section{Parameter Estimation Error Bound}\label{sec:theoreticalanalysis}
In this section, we theoretically analyze the parameter estimation performance of the solution obtained
by the MSMTFL algorithm.
To simplify the notations in the theoretical analysis, we assume that the number of samples for all the tasks
are the same. However, our theoretical analysis can be easily extended to
the case where the tasks have different sample sizes.

We first present a sub-Gaussian noise assumption which is very common in the analysis of sparse learning literature
\citep{zhang2012general,zhang2008adaptive,zhang2009some,zhang2010analysis,zhang2012multi}.
\begin{assumption}\label{assumption:subgaussian}
Let $\bar{W}=[\bar{\mathbf{w}}_1,\cdots,\bar{\mathbf{w}}_m]\in\mathbb{R}^{d\times m}$ be the underlying sparse weight matrix and
$\mathbf{y}_i=X_i\bar{\mathbf{w}}_i+\bm{\delta}_i,~\mathbb{E}\mathbf{y}_i=X_i\bar{\mathbf{w}}_i$, where
$\bm{\delta}_i\in\mathbb{R}^{n}$ is a random vector with all entries
$\delta_{ji}~(j\in\mathbb{N}_n,i\in\mathbb{N}_m)$ being independent sub-Gaussians: there exists
$\sigma>0$ such that $\forall j\in\mathbb{N}_n,i\in\mathbb{N}_{m},t\in\mathbb{R}$:
%$\mathbb{E}_{\delta_{ji}}\exp(t\delta_{ji})\leq \exp\left(\sigma^2t^2/2\right)$.
\begin{align}
\mathbb{E}_{\delta_{ji}}\exp(t\delta_{ji})\leq \exp\left(\frac{\sigma^2t^2}{2}\right).\nonumber
\end{align}
\end{assumption}

\begin{remark}
We call the random variable satisfying the condition in \AssumeRef{assumption:subgaussian} sub-Gaussian, since its moment generating function
is bounded by that of a zero mean Gaussian random variable. That is, if a normal random variable $x\sim N(0,\sigma^2)$, then we have:
\begin{align}
\mathbb{E}\exp(tx)&=\int_{-\infty}^{\infty}\exp(tx)\frac{1}{\sqrt{2\pi}\sigma}\exp\left(-\frac{x^2}{2\sigma^2}\right)dx\nonumber\\
&=\exp(\sigma^2t^2/2)\int_{-\infty}^{\infty}\frac{1}{\sqrt{2\pi}\sigma}\exp\left(-\frac{(x-\sigma^2t)^2}{2\sigma^2}\right)dx\nonumber\\
&=\exp(\sigma^2t^2/2).\nonumber
\end{align}
\end{remark}

\begin{remark}
Based on the Hoeffding's Lemma, for any random variable $x\in[a,b]$ and $\mathbb{E}x=0$,
we have $\mathbb{E}(\exp(tx))\leq\exp\left(\frac{t^2(b-a)^2}{8}\right)$.
Therefore, both zero mean Gaussian and zero mean bounded random variables are sub-Gaussians.
Thus, the sub-Gaussian noise assumption is more general than the Gaussian noise assumption
which is commonly used in the multi-task learning literature \citep{jalali2010dirty,lounici2009taking}.
\end{remark}

We next introduce the following sparse eigenvalue concept which is also common in the analysis of sparse learning literature
\citep{zhang2008sparsity,zhang2012general,zhang2009some,zhang2010analysis,zhang2012multi}.
\begin{definition}\label{def:sparseeigenvalue}
Given $1\leq k\leq d$, we define
\begin{align}
&\rho^+_i(k)=\sup_{\mathbf{w}}\left\{\frac{\|X_i\mathbf{w}\|^2}{n\|\mathbf{w}\|^2}:\|\mathbf{w}\|_0\leq k\right\},~\rho^+_{max}(k)=\max_{i\in\mathbb{N}_m}\rho^+_i(k), \nonumber\\
&\rho^-_i(k)=\inf_{\mathbf{w}}\left\{\frac{\|X_i\mathbf{w}\|^2}{n\|\mathbf{w}\|^2}:\|\mathbf{w}\|_0\leq k\right\},~\rho^-_{min}(k)=\min_{i\in\mathbb{N}_m}\rho^-_i(k).\nonumber
\end{align}
\end{definition}
\begin{remark}
$\rho^+_i(k)~(\rho^-_i(k))$ is in fact the maximum (minimum) eigenvalue of $(X_i)_{\mathcal{S}}^T(X_i)_{\mathcal{S}}/n$, where
$\mathcal{S}$ is a set satisfying $|\mathcal{S}|\leq k$ and $(X_i)_{\mathcal{S}}$ is a submatrix composed of the columns of $X_i$ indexed by $\mathcal{S}$.
In the MTL setting, we need to exploit the relations of $\rho^+_i(k)~(\rho^-_i(k))$ among multiple tasks.
\end{remark}

We present our parameter estimation error bound on MSMTFL in the following theorem:
\begin{theorem}\label{theorem:mainbound}
Let \AssumeRef{assumption:subgaussian} hold. Define $\bar{\mathcal{F}}_i=\{(j,i):\bar{w}_{ji}\neq 0\}$ and $\bar{\mathcal{F}}=\cup_{i\in\mathbb{N}_m}\bar{\mathcal{F}}_i$.
Denote $\bar{r}$ as the number of nonzero rows of $\bar{W}$.
We assume that
\begin{align}
&\forall (j,i)\in\bar{\mathcal{F}},\|\bar{\mathbf{w}}^j\|_1\geq2\theta\label{eq:noiselevel}\\
\mathrm{and}~&\frac{\rho^+_i(s)}{\rho^-_i(2\bar{r}+2s)}\leq 1+\frac{s}{2\bar{r}}, \label{eq:eigenvalueineq}
\end{align}
where $s$ is some integer satisfying $s\geq\bar{r}$.
If we choose $\lambda$ and $\theta$ such that for some $s\geq\bar{r}$:
\begin{align}
\lambda &\geq 12\sigma\sqrt{\frac{2\rho^+_{max}(1)\ln(2dm/\eta)}{n}},\label{eq:lambdacondition}\\
\theta &\geq \frac{11m\lambda}{\rho^-_{min}(2\bar{r}+s)},\label{eq:thetacondition}
\end{align}then the following parameter estimation error
bound holds with probability larger than $1-\eta$:
\begin{align}
&\|\hat{W}^{(\ell)}-\bar{W}\|_{2,1}\leq0.8^{\ell/2}\frac{9.1m\lambda\sqrt{\bar{r}}}{\rho^-_{min}(2\bar{r}+s)}+\frac{39.5m\sigma\sqrt{\rho^+_{max}(\bar{r})(7.4\bar{r}+2.7\ln(2/\eta))/n}}{\rho^-_{min}(2\bar{r}+s)},\label{eq:estimatebound}
\end{align}
where $\hat{W}^{(\ell)}$ is a solution of \EqRef{eq:relaxedlasso}.
\end{theorem}
\begin{remark}
\EqRef{eq:noiselevel} assumes that the $\ell_1$-norm of each nonzero row of $\bar{W}$ is away from zero. This requires
the true nonzero coefficients should be large enough, in order to distinguish them from the noise.
\EqRef{eq:eigenvalueineq} is called the sparse eigenvalue condition \citep{zhang2012multi}, which
requires the eigenvalue ratio $\rho^+_i(s)/\rho^-_i(s)$ to grow sub-linearly with respect to $s$.
Such a condition is very common in the analysis of sparse regularization \citep{zhang2008sparsity,zhang2009some} and it is
slightly weaker than the RIP condition \citep{candes2005decoding,huang2010benefit,zhang2012multi}.
\end{remark}
\begin{remark}
When $\ell=1$ (corresponds to Lasso), the first term of the right-hand side of \EqRef{eq:estimatebound} dominates the error bound in the order of
\begin{align}\label{eq:lassobound}
\|\hat{W}^{Lasso}-\bar{W}\|_{2,1}=O\left(m\sqrt{\bar{r}\ln(dm/\eta)/n}\right),
\end{align}
since $\lambda$ satisfies the condition in \EqRef{eq:lambdacondition}.
Note that the first term of the right-hand side of \EqRef{eq:estimatebound} shrinks exponentially as $\ell$ increases.
When $\ell$ is sufficiently large in the order of $O(\ln(m\sqrt{\bar{r}/n})+\ln\ln(dm))$,
this term tends to zero and we obtain the following parameter estimation error bound:
\begin{align}\label{eq:msmtflbound}
\|\hat{W}^{(\ell)}-\bar{W}\|_{2,1}=O\left(m\sqrt{\bar{r}/n+\ln(1/\eta)/n}\right).
\end{align}
\citet{jalali2010dirty} gave an $\ell_{\infty,\infty}$-norm error bound $\|\hat{W}^{Dirty}-\bar{W}\|_{\infty,\infty}=O\left(\sqrt{\ln(dm/\eta)/n}\right)$
as well as a sign consistency result between $\hat{W}$ and $\bar{W}$.
A direct comparison between these two bounds is difficult due to the use of different norms. On the other hand, the worst-case estimate of the
$\ell_{2,1}$-norm error bound of the algorithm in \citet{jalali2010dirty} is in the same order with \EqRef{eq:lassobound}, that is:
$\|\hat{W}^{Dirty}-\bar{W}\|_{2,1}=O\left(m\sqrt{\bar{r}\ln(dm/\eta)/n}\right)$.
%\begin{align}
%\|\hat{W}^{Dirty}-\bar{W}\|_{2,1}=O\left(m\sqrt{\bar{r}\ln(dm/\eta)/n}\right).\label{eq:dirtybound}
%\end{align}
When $dm$ is large and the ground truth has a large number of sparse rows (i.e., $\bar{r}$ is a small constant), the bound in \EqRef{eq:msmtflbound} is significantly better than the ones for the Lasso and Dirty model.
\end{remark}
\begin{remark}
\citet{jalali2010dirty} presented an $\ell_{\infty,\infty}$-norm parameter estimation
error bound and hence a sign consistency result can
be obtained. The results are derived
under the incoherence condition which is more restrictive than the RIP condition and
hence more restrictive than the sparse eigenvalue condition
in \EqRef{eq:eigenvalueineq}. From the viewpoint of the parameter estimation error, our
proposed algorithm can achieve a better bound under weaker conditions. Please refer to \citep{van2009conditions,zhang2009some,zhang2012multi} for more details about
the incoherence condition, the RIP condition, the sparse eigenvalue condition and their relationships.
\end{remark}
\begin{remark}
The capped-$\ell_1$ regularized formulation
in \citet{zhang2010analysis} is a special case of our formulation when $m=1$. However,
extending the analysis from the single task to the multi-task setting is nontrivial.
Different from previous work on multi-stage sparse learning which focuses on a single task~\citep{zhang2010analysis,zhang2012multi}, we study a more general multi-stage framework in the multi-task setting. We need to exploit the relationship among tasks, by using the relations of sparse eigenvalues $\rho^+_i(k)~(\rho^-_i(k))$ and treating the $\ell_1$-norm on each row of the weight matrix as a whole for consideration. Moreover, we simultaneously exploit the relations of each column and each row of the matrix.
\end{remark}

\section{Proof Sketch of \ThemRef{theorem:mainbound}}\label{sec:proof}
In this section, we present a proof sketch of \ThemRef{theorem:mainbound}. We first provide several important lemmas (detailed proofs
are available in the Appendix) and then complete the proof of \ThemRef{theorem:mainbound} based on these lemmas.

\begin{lemma}\label{lemma:wepsilonerror}
Let $\bar{\Upsilon}=[\bar{\bm{\epsilon}}_1,\cdots,\bar{\bm{\epsilon}}_m]$ with $\bar{\bm{\epsilon}}_i=[\bar{\epsilon}_{1i},\cdots,\bar{\epsilon}_{di}]^T=\frac{1}{n}X_i^T(X_i\bar{\mathbf{w}}_i-\mathbf{y}_i)~(i\in\mathbb{N}_m)$. Define $\bar{\mathcal{H}}\supseteq\bar{\mathcal{F}}$ such that $(j,i)\in\bar{\mathcal{H}}~(\forall i\in\mathbb{N}_m)$, provided there exists $(j,g)\in\bar{\mathcal{F}}$ ($\bar{\mathcal{H}}$ is a set consisting of the indices of all entries in the nonzero rows of $\bar{W}$). Under the conditions of \AssumeRef{assumption:subgaussian} and the notations of \ThemRef{theorem:mainbound}, the followings
hold with probability larger than $1-\eta$:
\begin{align}
&\|\bar{\Upsilon}\|_{\infty,\infty}\leq\sigma\sqrt{\frac{2\rho^+_{max}(1)\ln(2dm/\eta)}{n}},\label{eq:epsilonerror}\\
&\|\bar{\Upsilon}_{\bar{\mathcal{H}}}\|^2_F\leq m\sigma^2\rho^+_{max}(\bar{r})(7.4\bar{r}+2.7\ln(2/\eta))/n.\label{eq:epsilonHerror}
\end{align}
\end{lemma}

\LemmaRef{lemma:wepsilonerror} gives bounds on the residual correlation ($\bar{\Upsilon}$) with respect to $\bar{W}$. We note that \EqRef{eq:epsilonerror}
and \EqRef{eq:epsilonHerror} are closely related to the assumption on $\lambda$ in \EqRef{eq:lambdacondition} and
the second term of the right-hand side of \EqRef{eq:estimatebound} (error bound), respectively. This lemma provides
a fundamental basis for the proof of \ThemRef{theorem:mainbound}.

\begin{lemma}\label{lemma:wgbound}
Use the notations of \LemmaRef{lemma:wepsilonerror} and consider $\mathcal{G}_i\subseteq\mathbb{N}_{d}\times\{i\}$
such that $\bar{\mathcal{F}}_i\cap\mathcal{G}_i=\emptyset~(i\in\mathbb{N}_m)$. Let $\hat{W}=\hat{W}^{(\ell)}$
be a solution of \EqRef{eq:relaxedlasso} and $\Delta\hat{W}=\hat{W}-\bar{W}$.
Denote $\hat{\bm{\lambda}}_i=\hat{\bm{\lambda}}_i^{(\ell-1)}=[\lambda^{(\ell-1)}_{1},\cdots,\lambda^{(\ell-1)}_{d}]^T$. Let
$\hat\lambda_{\mathcal{G}_i}=\min_{(j,i)\in\mathcal{G}_i}\hat\lambda_{ji},~\hat\lambda_{\mathcal{G}}=\min_{i\in\mathcal{G}_i}\hat\lambda_{\mathcal{G}_i}$ and $\hat\lambda_{0i}=\max_j\hat\lambda_{ji},~\hat\lambda_{0}=\max_{i}\hat\lambda_{0i}$. If $2\|\bar{\bm{\epsilon}}_i\|_\infty<\hat\lambda_{\mathcal{G}_i}$,
then the following inequality holds at any stage $\ell\geq1$:
\begin{align}
\sum_{i=1}^m\sum_{(j,i)\in\mathcal{G}_i}|\hat{w}_{ji}^{(\ell)}|\leq
\frac{2\|\bar{\Upsilon}\|_{\infty,\infty}+\hat\lambda_0}{\hat\lambda_{\mathcal{G}}-2\|\bar{\Upsilon}\|_{\infty,\infty}}\sum_{i=1}^m\sum_{(j,i)\in\mathcal{G}_i^c}|\Delta\hat{w}_{ji}^{(\ell)}|. \nonumber
\end{align}
\end{lemma}

Denote $\mathcal{G}=\cup_{i\in\mathbb{N}_m}\mathcal{G}_i,~\bar{\mathcal{F}}=\cup_{i\in\mathbb{N}_m}{\bar{\mathcal{F}}_i}$ and notice that $\bar{\mathcal{F}}\cap\mathcal{G}=\emptyset\Rightarrow\Delta\hat{W}^{(\ell)}=\hat{W}^{(\ell)}$. \LemmaRef{lemma:wgbound} says that $\|\Delta\hat{W}_{\mathcal{G}}^{(\ell)}\|_{1,1}=\|\hat{W}_{\mathcal{G}}^{(\ell)}\|_{1,1}$
is upper bounded in terms of $\|\Delta\hat{W}_{\mathcal{G}^c}^{(\ell)}\|_{1,1}$, which indicates that the error of the estimated coefficients locating
outside of $\bar{\mathcal{F}}$ should be small enough. This provides an intuitive
explanation why the parameter estimation error of our algorithm can be small.

\begin{lemma}\label{lemma:deltawbound}
Using the notations of \LemmaRef{lemma:wgbound}, we denote $\mathcal{G}=\mathcal{G}_{(\ell)}=\bar{\mathcal{H}}^c\cap\{(j,i):\hat\lambda_{ji}^{(\ell-1)}=\lambda\}=\cup_{i\in\mathbb{N}_m}\mathcal{G}_i$ with $\bar{\mathcal{H}}$ being defined as in \LemmaRef{lemma:wepsilonerror} and $\mathcal{G}_i\subseteq\mathbb{N}_d\times\{i\}$. Let $\mathcal{J}_i$ be the indices of the largest $s$ coefficients
(in absolute value) of $\hat{\mathbf{w}}_{\mathcal{G}_i}$, $\mathcal{I}_i=\mathcal{G}_i^c\cup\mathcal{J}_i$,
$\mathcal{I}=\cup_{i\in\mathbb{N}_m}{\mathcal{I}_i}$ and $\bar{\mathcal{F}}=\cup_{i\in\mathbb{N}_m}{\bar{\mathcal{F}}_i}$.
Then, the following inequalities hold at any stage $\ell\geq1$:
\begin{align}
%&\|\Delta\hat{W}^{(\ell)}\|_{2,1}\leq\left(1+1.5\sqrt{\frac{2\bar{r}}{s}}\right)\|\Delta\hat{W}_{\mathcal{I}}^{(\ell)}\|_{2,1}\leq\frac{\left(1+1.5\sqrt{\frac{2\bar{r}}{s}}\right)\sqrt{8m\left(4\|\bar{\Upsilon}_{\mathcal{G}_{(\ell)}^c}\|_F^2+\sum_{(j,i)\in\bar{\mathcal{F}}}(\hat\lambda_{ji}^{(\ell-1)})^2\right)}}{\rho^-_{min}(2\bar{r}+s)}.\label{eq:deltawbound}\\
&\|\Delta\hat{W}^{(\ell)}\|_{2,1}\leq\frac{\left(1+1.5\sqrt{\frac{2\bar{r}}{s}}\right)\sqrt{8m\left(4\|\bar{\Upsilon}_{\mathcal{G}_{(\ell)}^c}\|_F^2+\sum_{(j,i)\in\bar{\mathcal{F}}}(\hat\lambda_{ji}^{(\ell-1)})^2\right)}}{\rho^-_{min}(2\bar{r}+s)},\label{eq:deltawbound}\\
&\|\Delta\hat{W}^{(\ell)}\|_{2,1}\leq\frac{9.1m\lambda\sqrt{\bar{r}}}{\rho^-_{min}(2\bar{r}+s)}.\label{eq:w21error}
\end{align}
\end{lemma}

\LemmaRef{lemma:deltawbound} is established based on \LemmaRef{lemma:wgbound},
by considering the relationship between \EqRef{eq:lambdacondition} and \EqRef{eq:epsilonerror},
and the specific definition of $\mathcal{G}=\mathcal{G}_{(\ell)}$. \EqRef{eq:deltawbound} provides a parameter estimation error bound in terms of $\ell_{2,1}$-norm by $\|\bar{\Upsilon}_{\mathcal{G}_{(\ell)}^c}\|_F^2$
and the regularization parameters $\hat\lambda_{ji}^{(\ell-1)}$ (see the definition of $\hat\lambda_{ji}~(\hat\lambda_{ji}^{(\ell-1)})$ in \LemmaRef{lemma:wgbound}).
This is the result directly used in the proof of \ThemRef{theorem:mainbound}.
\EqRef{eq:w21error} states that the error bound is upper bounded in terms of $\lambda$, the right-hand side of which constitutes the shrinkage part of the error bound in \EqRef{eq:estimatebound}.

\begin{lemma}\label{lemma:lambdadecomp}
Let $\hat\lambda_{ji}=\lambda I\left(\|\hat{\mathbf{w}}^j\|_1<\theta,j\in\mathbb{N}_d\right),\forall i\in\mathbb{N}_m$ with some $\hat{W}\in\mathbb{R}^{d\times m}$.
$\bar{\mathcal{H}}\supseteq\bar{\mathcal{F}}$ is defined in \LemmaRef{lemma:wepsilonerror}. Then under the condition of \EqRef{eq:noiselevel}, we have:
\begin{align}
&\sum_{(j,i)\in\bar{\mathcal{F}}}\hat\lambda_{ji}^2\leq\sum_{(j,i)\in\bar{\mathcal{H}}}\hat\lambda_{ji}^2\leq m\lambda^2\|\bar{W}_{\bar{\mathcal{H}}}-\hat{W}_{\bar{\mathcal{H}}}\|_{2,1}^2/\theta^2.\nonumber
\end{align}
\end{lemma}

\LemmaRef{lemma:lambdadecomp} establishes an upper bound of $\sum_{(j,i)\in\bar{\mathcal{F}}}\hat\lambda_{ji}^2$ by
$\|\bar{W}_{\bar{\mathcal{H}}}-\hat{W}_{\bar{\mathcal{H}}}\|_{2,1}^2$, which is critical for building the recursive relationship between
$\|\hat{W}^{(\ell)}-\bar{W}\|_{2,1}$ and $\|\hat{W}^{(\ell-1)}-\bar{W}\|_{2,1}$ in the proof of \ThemRef{theorem:mainbound}. This recursive relation is crucial for the
shrinkage part of the error bound in \EqRef{eq:estimatebound}.

\subsection{Proof of \ThemRef{theorem:mainbound}}\label{sec:proof.mainbound}
\begin{proof}
For notational simplicity, we denote the right-hand side of \EqRef{eq:epsilonHerror} as:
\begin{align}\label{eq:uvalue}
u=m\sigma^2\rho^+_{max}(\bar{r})(7.4\bar{r}+2.7\ln(2/\eta))/n.
\end{align}
Based on $\bar{\mathcal{H}}\subseteq\mathcal{G}_{(\ell)}^c$, \LemmaRef{lemma:wepsilonerror} and \EqRef{eq:lambdacondition}, the followings hold with probability larger than $1-\eta$:
\begin{align}
\|\bar{\Upsilon}_{\mathcal{G}_{(\ell)}^c}\|_F^2 &= \|\bar{\Upsilon}_{\bar{\mathcal{H}}}\|_F^2+\|\bar{\Upsilon}_{\mathcal{G}_{(\ell)}^c\setminus\bar{\mathcal{H}}}\|_F^2\nonumber\\
&\leq u+|\mathcal{G}_{(\ell)}^c\setminus\bar{\mathcal{H}}|\|\bar{\Upsilon}\|_{\infty,\infty}^2\nonumber\\
&\leq u+\lambda^2|\mathcal{G}_{(\ell)}^c\setminus\bar{\mathcal{H}}|/144\nonumber\\
&\leq u+(1/144)m\lambda^2\theta^{-2}\|\hat{W}^{(\ell-1)}_{\mathcal{G}_{(\ell)}^c\setminus\bar{\mathcal{H}}}-\bar{W}_{\mathcal{G}_{(\ell)}^c\setminus\bar{\mathcal{H}}}\|_{2,1}^2,\label{eq:epsilondecomp}
\end{align}
where the last inequality follows from
\begin{align}
\forall &(j,i)\in\mathcal{G}_{(\ell)}^c\setminus\bar{\mathcal{H}},\|(\hat{\mathbf{w}}^{(\ell-1)})^j\|_1^2/\theta^2=\|(\hat{\mathbf{w}}^{(\ell-1)})^j-\bar{\mathbf{w}}^j\|_1^2/\theta^2\geq1\nonumber\\
\Rightarrow|&\mathcal{G}_{(\ell)}^c\setminus\bar{\mathcal{H}}|\leq m\theta^{-2}\|\hat{W}^{(\ell-1)}_{\mathcal{G}_{(\ell)}^c\setminus\bar{\mathcal{H}}}-\bar{W}_{\mathcal{G}_{(\ell)}^c\setminus\bar{\mathcal{H}}}\|_{2,1}^2.\nonumber
\end{align}
According to \EqRef{eq:deltawbound}, we have:
\begin{align}
&\|\hat{W}^{(\ell)}-\bar{W}\|_{2,1}^2=\|\Delta\hat{W}^{(\ell)}\|_{2,1}^2\nonumber\\
&\leq\frac{8m\left(1+1.5\sqrt{\frac{2\bar{r}}{s}}\right)^2\left(4\|\bar{\Upsilon}_{\mathcal{G}_{(\ell)}^c}\|_F^2+\sum_{(j,i)\in\bar{\mathcal{F}}}(\hat\lambda_{ji}^{(\ell-1)})^2\right)}{(\rho^-_{min}(2\bar{r}+s))^2}\nonumber\\
&\leq\frac{78m\left(4u+(37/36)m\lambda^2\theta^{-2}\left\|\hat{W}^{(\ell-1)}-\bar{W}\right\|_{2,1}^2\right)}{(\rho^-_{min}(2\bar{r}+s))^2}\nonumber\\
&\leq\frac{312m u}{(\rho^-_{min}(2\bar{r}+s))^2}+0.8\left\|\hat{W}^{(\ell-1)}-\bar{W}\right\|_{2,1}^2\nonumber\\
&\leq\cdots\leq0.8^{\ell}\left\|\hat{W}^{(0)}-\bar{W}\right\|_{2,1}^2+\frac{312m u}{(\rho^-_{min}(2\bar{r}+s))^2}\frac{1-0.8^{\ell}}{1-0.8}\nonumber\\
&\leq0.8^{\ell}\frac{9.1^2m^2\lambda^2\bar{r}}{(\rho^-_{min}(2\bar{r}+s))^2}+\frac{1560m u}{(\rho^-_{min}(2\bar{r}+s))^2}.\nonumber
\end{align}
In the above derivation, the first inequality is due to \EqRef{eq:deltawbound}; the second inequality is due
to the assumption $s\geq\bar{r}$ in \ThemRef{theorem:mainbound}, \EqRef{eq:epsilondecomp} and \LemmaRef{lemma:lambdadecomp}; the third inequality is due to \EqRef{eq:thetacondition}; the last inequality follows from \EqRef{eq:w21error} and $1-0.8^{\ell}\leq 1~(\ell\geq1)$.
Thus, following the inequality $\sqrt{a+b}\leq\sqrt{a}+\sqrt{b}~(\forall a,b\geq0)$, we obtain:
\begin{align}
&\|\hat{W}^{(\ell)}-\bar{W}\|_{2,1}\leq0.8^{\ell/2}\frac{9.1m\lambda\sqrt{\bar{r}}}{\rho^-_{min}(2\bar{r}+s)}+\frac{39.5\sqrt{mu}}{\rho^-_{min}(2\bar{r}+s)}.\nonumber
\end{align}
Substituting \EqRef{eq:uvalue} into the above inequality, we verify \ThemRef{theorem:mainbound}.
\end{proof}
\begin{remark}
The assumption $s\geq\bar{r}$ used in the above proof indicates that at each stage,
the zero entries of $\hat{W}^{(\ell)}$ should be greater than $m\bar{r}$
(see definition of $s$ in \LemmaRef{lemma:deltawbound}). This requires the solution obtained by \AlgRef{alg:msmtfl}
at each stage is sparse, which is consistent with the sparsity of $\bar{W}$ in \AssumeRef{assumption:subgaussian}.
\end{remark}

\section{Experiments}\label{sec:experiments}
In this section, we present empirical studies on both synthetic and real-world data sets.
In the synthetic data experiments, we present the performance of the MSMTFL algorithm in terms of the parameter estimation error.
In the real-world data experiments, we show the performance of the
MSMTFL algorithm in terms of the prediction error.

\subsection{Competing Algorithms}\label{sec:experiments.algorithm}
%Our intent is not to conduct exhaustive comparisons with various multi-task feature learning algorithms, but
%to mainly show the advantage of our non-convex formulation over the convex ones.
We present the empirical studies by comparing our proposed MSMTFL algorithm with
three competing multi-task feature learning algorithms:
$\ell_1$-norm multi-task feature learning algorithm (Lasso), $\ell_{1,2}$-norm multi-task feature
learning algorithm (L1,2) \citep{obozinski2006multi} and dirty model
multi-task feature learning algorithm (DirtyMTL) \citep{jalali2010dirty}.
In our experiments, we employ the quadratic loss function in \EqRef{eq:lossfunction} for
all the compared algorithms.

\subsection{Synthetic Data Experiments}\label{sec:experiments.synthetic}
We generate synthetic data by setting the number of tasks as
$m$ and each task has $n$ samples which are of dimensionality $d$;
each element of the data matrix $X_i\in\mathbb{R}^{n\times d}~(i\in\mathbb{N}_m)$ for the $i$-th task
is sampled i.i.d. from the Gaussian distribution $N(0,1)$ and we then normalize all columns to length $1$;
each entry of the underlying true weight $\bar{W}\in\mathbb{R}^{d\times m}$ is sampled i.i.d. from the uniform
distribution in the interval $[-10,10]$; we randomly set $90\%$ rows of $\bar{W}$ as zero vectors and
$80\%$ elements of the remaining nonzero entries as zeros; each entry of the noise $\bm{\delta}_i\in\mathbb{R}^{n}$
is sampled i.i.d. from the Gaussian distribution $N(0,\sigma^2)$; the responses are computed as $\mathbf{y}_i=X_i\bar{\mathbf{w}}_i+\bm{\delta}_i~(i\in\mathbb{N}_m)$.

We first report the averaged parameter estimation error $\|\hat{W}-\bar{W}\|_{2,1}$ vs. Stage ($\ell$) plots for MSMTFL (\FigRef{fig:multistage_illustration}).
We observe that the error decreases as $\ell$ increases, which shows the advantage of our proposed algorithm over Lasso.
This is consistent with the theoretical result in \ThemRef{theorem:mainbound}.
Moreover, the parameter estimation error decreases quickly and converges in a few stages.

We then report the averaged parameter estimation error $\|\hat{W}-\bar{W}\|_{2,1}$ in comparison with four algorithms in different parameter settings (\FigRef{fig:toy_W21error}).
For a fair comparison, we compare the smallest estimation errors of the four algorithms in all the parameter settings \citep{zhang2009some,zhang2010analysis}. As expected,
the parameter estimation error of the MSMTFL algorithm is the smallest among the four algorithms. This empirical result demonstrates the effectiveness of the MSMTFL algorithm. We also have the following observations: (a) When $\lambda$ is large enough, all four algorithms tend
to have the same parameter estimation error. This is reasonable, because the solutions $\hat{W}$'s obtained by the four algorithms are all zero matrices, when $\lambda$ is very large.
(b) The performance of the MSMTFL algorithm is similar for different $\theta$'s, when $\lambda$ exceeds a certain value.

\begin{figure}[!ht]\vspace{-0.0cm}
\begin{minipage}[c]{1.0\linewidth}
\centering
\includegraphics[width=0.5\linewidth]{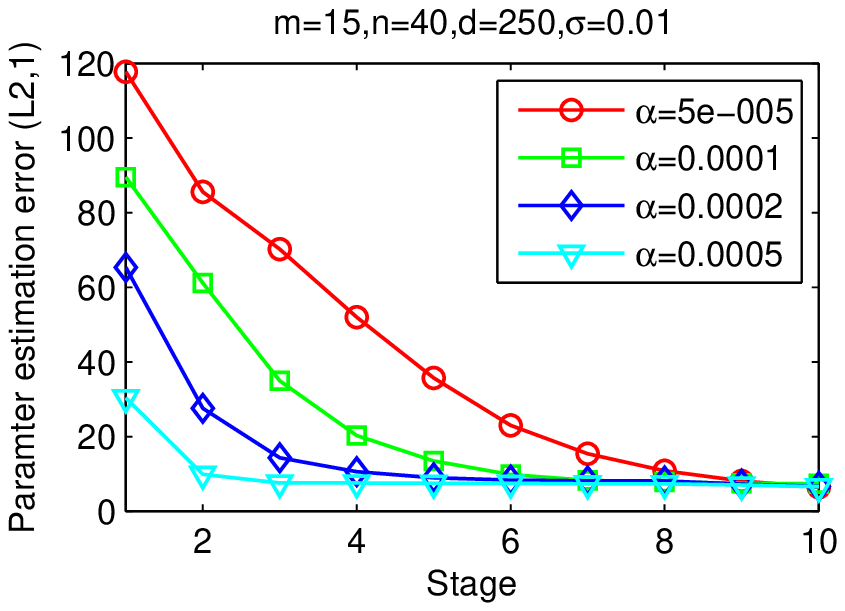}
\includegraphics[width=0.5\linewidth]{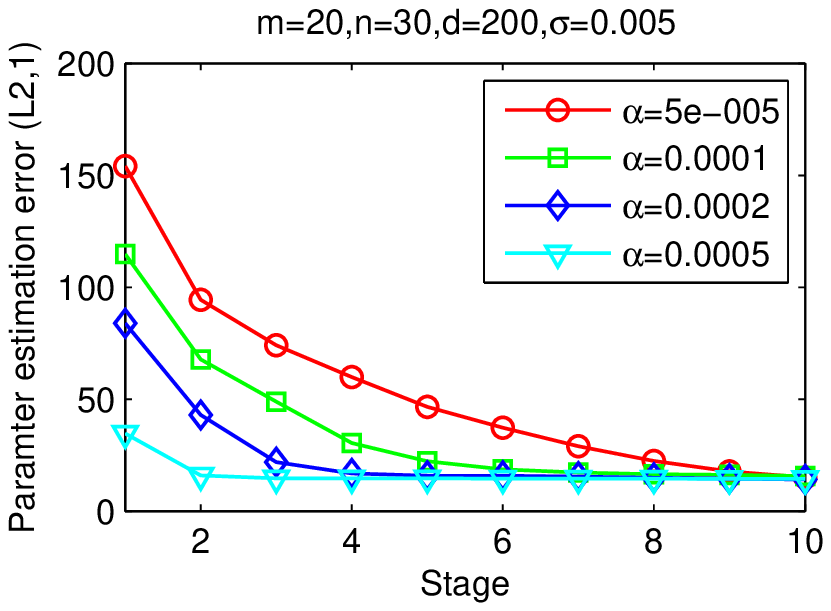}
\includegraphics[width=0.5\linewidth]{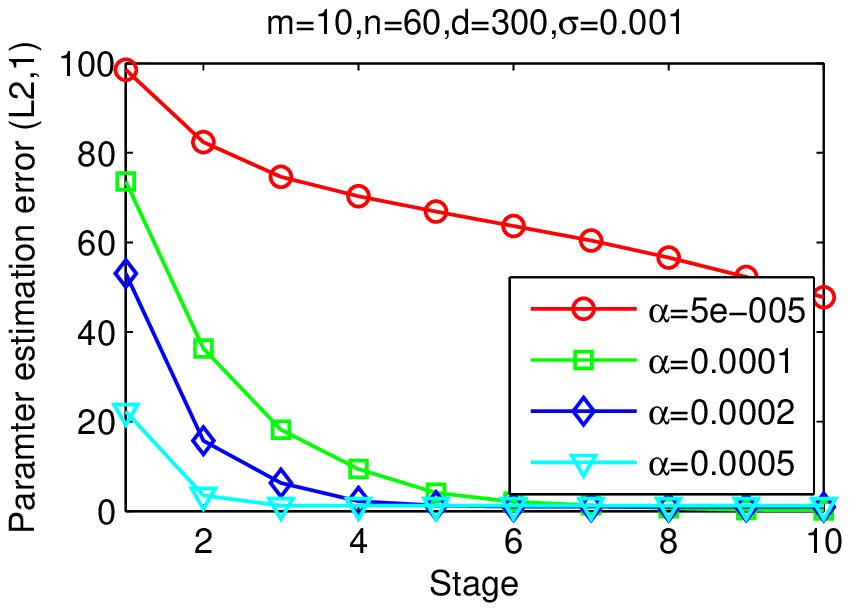}
\end{minipage}
\vspace{-0.0cm}\caption{Averaged parameter estimation error $\|\hat{W}-\bar{W}\|_{2,1}$ vs. Stage ($\ell$)
plots for MSMTFL on the synthetic data set (averaged over 10 runs). Here we set $\lambda=\alpha\sqrt{\ln(dm)/n},~\theta=50m\lambda$.
Note that $\ell=1$ corresponds to Lasso; the results show the stage-wise improvement over Lasso.}
\label{fig:multistage_illustration}\vspace{-0.0cm}
\end{figure}

\begin{figure}[!ht]\vspace{-0.0cm}
\begin{minipage}[c]{1.0\linewidth}
%\centering
\hspace{3.5cm}\includegraphics[width=0.5\linewidth]{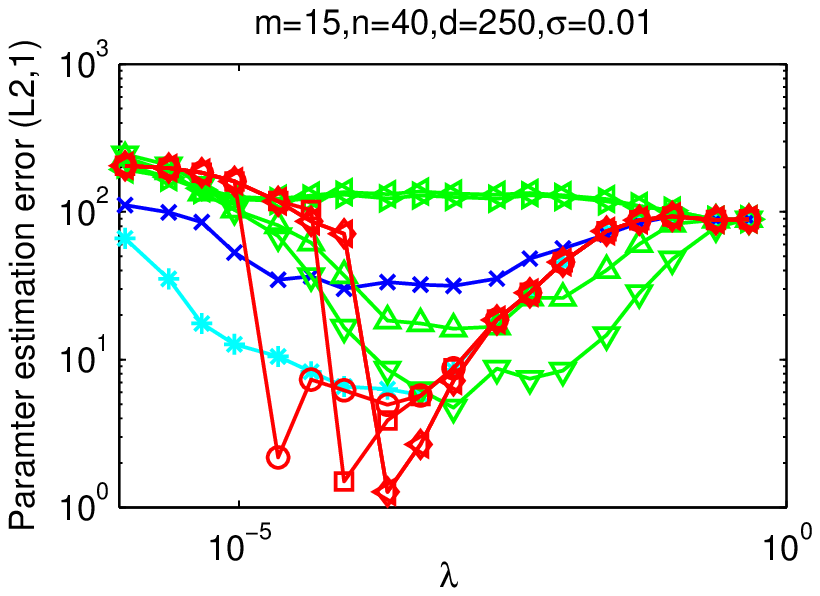}
\end{minipage}
\begin{minipage}[c]{1.0\linewidth}
%\centering
\hspace{3.5cm}\includegraphics[width=0.5\linewidth]{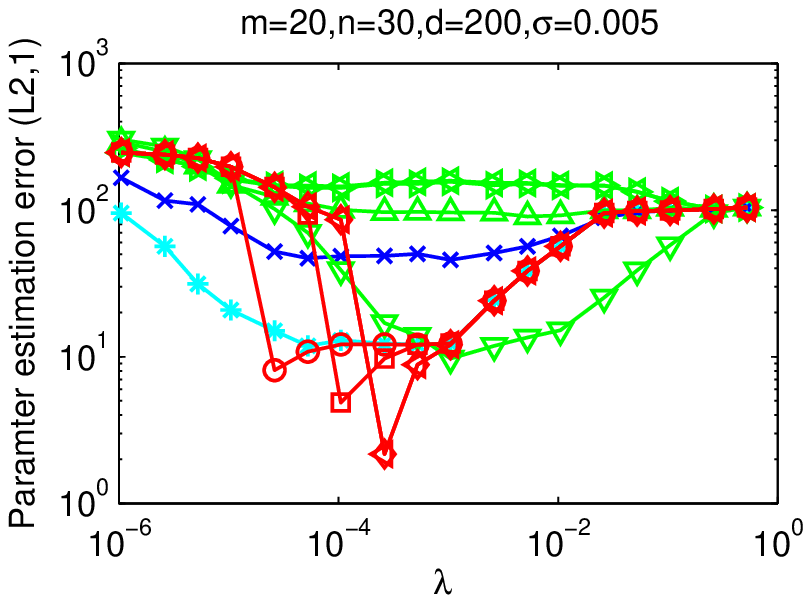}
\end{minipage}
\begin{minipage}[c]{1.0\linewidth}
%\centering
\hspace{3.5cm}\includegraphics[width=0.7\linewidth]{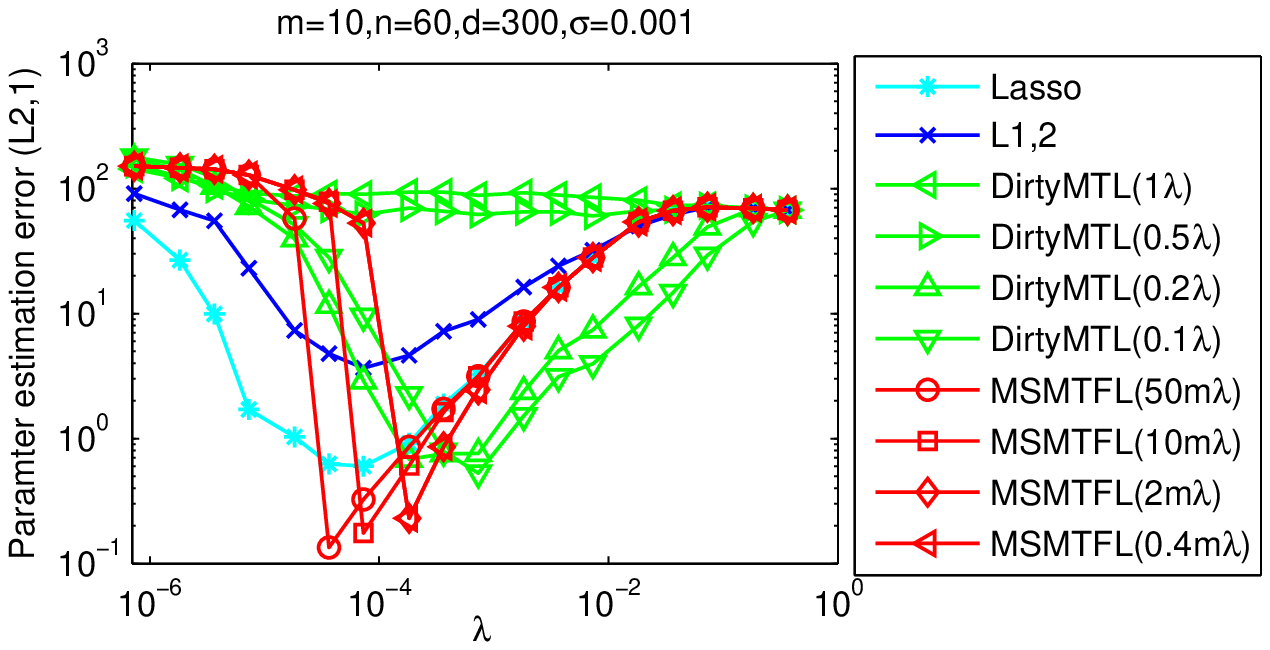}
\end{minipage}
\vspace{-0.0cm}\caption{Averaged parameter estimation error $\|\hat{W}-\bar{W}\|_{2,1}$ vs. $\lambda$
plots on the synthetic data set (averaged over 10 runs).  MSMTFL has
the smallest parameter estimation error among the four algorithms. Both DirtyMTL and MSMTFL have two parameters; we set
$\lambda_s/\lambda_b=1,0.5,0.2,0.1$ for DirtyMTL ($1/m\leq\lambda_s/\lambda_b\leq1$ was adopted in
\citet{jalali2010dirty}) and $\theta/\lambda=50m,10m,2m,0.4m$ for MSMTFL.} \label{fig:toy_W21error}\vspace{-0.0cm}
\end{figure}

\subsection{Real-World Data Experiments}\label{sec:experiments.real}
We conduct experiments on two real-world data sets: MRI and Isolet data sets.

The MRI data set is collected from the ANDI
database,
which contains 675 patients' MRI data preprocessed
using FreeSurfer\footnote{\url{www.loni.ucla.edu/ADNI/}}. The MRI
data include 306 features and the response (target) is the Mini Mental
State Examination (MMSE) score coming from 6 different time points: M06,
M12, M18, M24, M36, and M48. We remove the samples which fail the
MRI quality controls and have missing entries. Thus, we have 6 tasks
with each task corresponding to
a time point and the sample sizes corresponding to 6 tasks are
648, 642, 293, 569, 389 and 87, respectively.

The Isolet data set\footnote{\url{www.zjucadcg.cn/dengcai/Data/data.html}}
is collected from 150 speakers who
speak the name of each English letter of the alphabet
twice. Thus, there are 52 samples from each speaker. The speakers are
grouped into 5 subsets which respectively include 30 similar speakers,
and the subsets are named Isolet1, Isolet2, Isolet3, Isolet4, and Isolet5.
Thus, we naturally have 5 tasks with
each task corresponding to a subset. The 5 tasks respectively
have 1560, 1560, 1560, 1558, and 1559 samples\footnote{Three samples are historically missing.},
where each sample includes 617 features and the response is the English letter label (1-26).

\begin{table*}[ht]\vspace{-0.0cm}
\caption{Comparison of four multi-task feature learning algorithms on
the MRI data set in terms of averaged nMSE and
aMSE (standard deviation), which are averaged over 10 random splittings.}\vspace{-0.0cm} \label{tab:MRIdata}
\begin{center}
\small{
\begin{tabular}{c|c||c c c c } \hline\hline
measure &traning ratio &Lasso &L1,2 &DirtyMTL &MSMTFL\\
\hline
\multirow{4}{*}{nMSE}&$0.15$ &0.6651(0.0280) &0.6633(0.0470) &0.6224(0.0265) &\textbf{0.5539(0.0154)}\\
\multirow{4}{*}{}&$0.20$ &0.6254(0.0212) &0.6489(0.0275) &0.6140(0.0185) &\textbf{0.5542(0.0139)}\\
\multirow{4}{*}{}&$0.25$ &0.6105(0.0186) &0.6577(0.0194) &0.6136(0.0180) &\textbf{0.5507(0.0142)}\\
\hline\hline
\multirow{4}{*}{aMSE}&$0.15$ &0.0189(0.0008) &0.0187(0.0010) &0.0172(0.0006) &\textbf{0.0159(0.0004)}\\
\multirow{4}{*}{}&$0.20$ &0.0179(0.0006) &0.0184(0.0005) &0.0171(0.0005) &\textbf{0.0161(0.0004)}\\
\multirow{4}{*}{}&$0.25$ &0.0172(0.0009) &0.0183(0.0006) &0.0167(0.0008) &\textbf{0.0157(0.0006)}\\
\hline\hline
\end{tabular}}
\end{center}\vspace{-0.0cm}
\end{table*}

In the experiments, we treat the MMSE and letter labels as the regression values
for the MRI data set and the Isolet data set, respectively. For both data sets, we randomly
extract the training samples from each task with different training
ratios ($15\%,20\%$ and $25\%$) and use the rest of samples to form the test set.
We evaluate the four multi-task feature learning algorithms in terms of normalized mean squared error (nMSE) and averaged
means squared error (aMSE), which are commonly used in multi-task learning problems \citep{zhang2010multi,zhou2011clustered,gong2012robust}.
For each training ratio, both nMSE and aMSE are averaged over 10 random splittings of training and test
sets and the standard deviation is also shown.
All parameters of the four algorithms are tuned via 3-fold cross validation.

\begin{figure}[!ht]\vspace{-0.0cm}
\begin{minipage}[c]{1.0\linewidth}
\centering
\includegraphics[width=.49\linewidth]{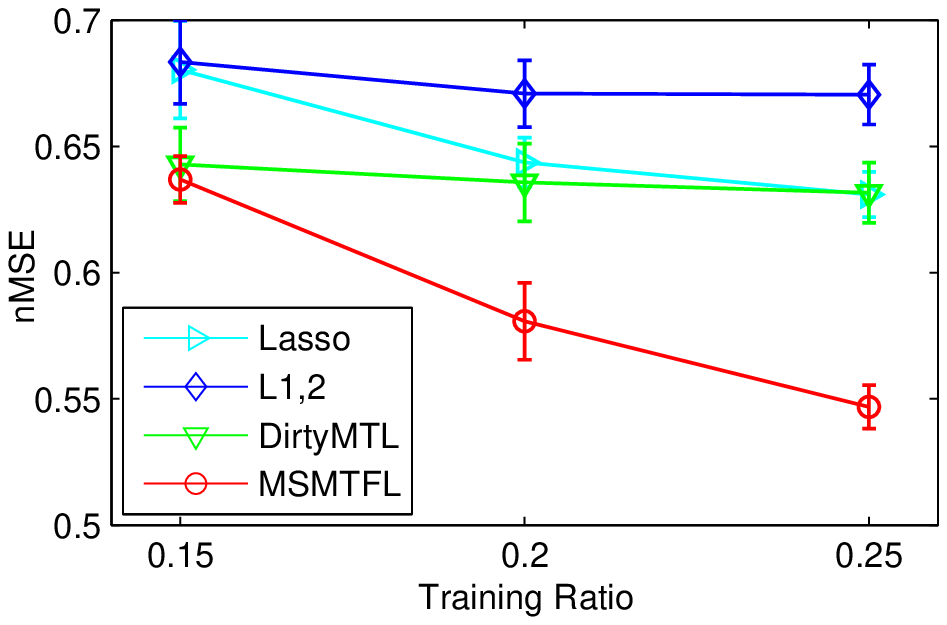}
\includegraphics[width=.49\linewidth]{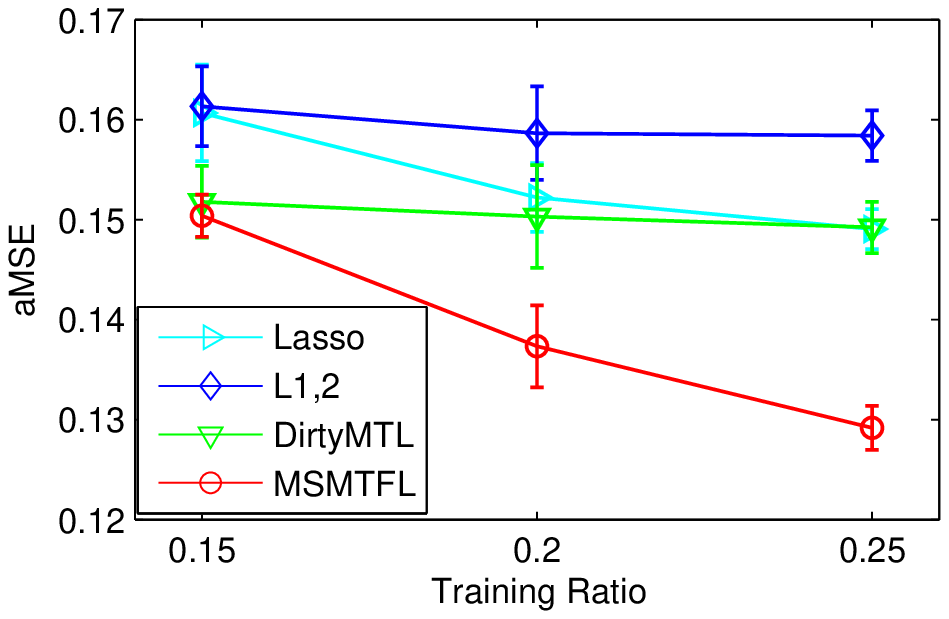}\\
\end{minipage}
\vspace{-0.0cm}\caption{Averaged test error (nMSE and aMSE) vs. training
ratio plots on the Isolet data set. The results are averaged over 10 random splittings.} \label{fig:Isolet_testerror}\vspace{-0.0cm}
\end{figure}

\TabRef{tab:MRIdata} and \FigRef{fig:Isolet_testerror} show the experimental results in terms of
averaged nMSE (aMSE) and the standard deviation. From these results, we observe
that: (a) Our proposed MSMTFL algorithm outperforms all the competing feature learning algorithms on
both data sets, with the smallest regression errors (nMSE and aMSE) as well as the smallest standard deviations.
(b) On the MRI data set, the MSMTFL algorithm performs well even in the case of a small training ratio.
The performance for the $15\%$ training ratio is comparable to that for the $25\%$ training ratio.
(c) On the Isolet data set, when the training ratio increases from $15\%$ to $25\%$, the performance of the MSMTFL algorithm increases and
the superiority of the MSMTFL algorithm over
the other three algorithms is more significant. Our results demonstrate the effectiveness of the proposed algorithm.

\section{Conclusions}\label{sec:conclusions}
In this paper, we propose a non-convex formulation for multi-task feature
learning, which learns the specific features of each task as well as the common features shared among tasks.
The non-convex formulation adopts the capped-$\ell_1$,$\ell_1$ regularizer to better approximate the $\ell_0$-type one than
the commonly used convex regularizer. To solve the non-convex optimization problem, we propose a Multi-Stage Multi-Task Feature Learning (MSMTFL) algorithm
and provide intuitive interpretations from several aspects. We also present a detailed convergence analysis and discuss the reproducibility issue for the proposed algorithm.
Specifically, we show that, under a mild condition, the solution generated by MSMTFL is unique. Although the solution may not be globally optimal, we theoretically show that
it has good performance in terms of the parameter estimation error bound. Experimental results on both synthetic and real-world data sets
demonstrate the effectiveness of our proposed MSMTFL algorithm in comparison with the state of the art multi-task feature
learning algorithms.

In our future work, we will explore the conditions under which
a globally optimal solution of the proposed formulation can be obtained by the MSMTFL algorithm.
We will also focus on a general non-convex regularization
framework for multi-task learning settings (involving different loss functions and non-convex regularization terms)
and derive theoretical bounds.

% Acknowledgements should go at the end, before appendices and references

%\acks{We would like to acknowledge support for this work
%from ...}


\newpage

\appendix
\section*{Appendix}
\label{appendix:proof}

In this appendix, we provide detailed proofs for Lemmas \ref{lemma:wepsilonerror} to \ref{lemma:lambdadecomp}.
In our proofs, we use several lemmas (summarized  in part B) from \citet{zhang2010analysis}.

We first introduce some notations used in the proof. Define
\begin{align}
\pi_i(k_i,s_i)=\sup_{\mathbf{v}\in\mathbb{R}^{k_i},\mathbf{u}\in\mathbb{R}^{s_i},\mathcal{I}_i,\mathcal{J}_i}\frac{\mathbf{v}^TA^{(i)}_{\mathcal{I}_i,\mathcal{J}_i}\mathbf{u}\|\mathbf{v}\|}{\mathbf{v}^TA^{(i)}_{\mathcal{I}_i,\mathcal{I}_i}\mathbf{v}\|\mathbf{u}\|_{\infty}},
\end{align}
where $s_i+k_i\leq d$ with $s_i,k_i\geq 1$; $\mathcal{I}_i$ and $\mathcal{J}_i$ are \emph{disjoint} subsets of $\mathbb{N}_{d}$ with $k_i$ and $s_i$ elements respectively (with some abuse of notation, we also let $\mathcal{I}_i$ be a subset of $\mathbb{N}_{d}\times\{i\}$, depending on the context.); $A^{(i)}_{\mathcal{I}_i,\mathcal{J}_i}$ is a sub-matrix of $A_i=n^{-1}X_i^TX_i\in\mathbb{R}^{d\times d}$ with rows indexed by $\mathcal{I}_i$ and columns indexed by $\mathcal{J}_i$.

We let $\mathbf{w}_{\mathcal{I}_i}$ be a $d\times 1$ vector with the $j$-th entry being $w_{ji}$, if $(j,i)\in\mathcal{I}_i$, and $0$, otherwise.
We also let $W_{\mathcal{I}}$ be a $d\times m$ matrix with $(j,i)$-th entry being $w_{ji}$, if $(j,i)\in\mathcal{I}$, and $0$, otherwise.

\section*{A. Proofs of Lemmas \ref{lemma:wepsilonerror} to \ref{lemma:lambdadecomp}}

\section*{A.1. Proof of \LemmaRef{lemma:wepsilonerror}}
\begin{proof}
For the $j$-th entry of $\bar{\bm{\epsilon}}_{i}~(j\in\mathbb{N}_{d})$:
\begin{align}
|\bar{\epsilon}_{ji}|&=\frac{1}{n}\left|\left(\mathbf{x}^{(i)}_j\right)^T(X_i\bar{\mathbf{w}}_i-\mathbf{y}_i)\right|=\frac{1}{n}\left|\left(\mathbf{x}^{(i)}_j\right)^T\bm{\delta}_i\right|,\nonumber
\end{align}
where $\mathbf{x}^{(i)}_j$ is the $j$-th column of $X_i$.
We know that the entries of $\bm{\delta}_i$ are
independent sub-Gaussian random variables, and $\|1/n\mathbf{x}^{(i)}_j\|^2=\|\mathbf{x}^{(i)}_j\|^2/n^2\leq\rho^+_i(1)/n$. According to \LemmaRef{lemma:subgaussianproperty}, we have $\forall t>0$:
\begin{align}
\Pr(|\bar{\epsilon}_{ji}|\geq t)&\leq 2\exp(-nt^2/(2\sigma^2\rho^+_i(1)))\leq 2\exp(-nt^2/(2\sigma^2\rho^+_{max}(1))).\nonumber
\end{align}
Thus we obtain:
\begin{align}
\Pr(\|\bar{\Upsilon}\|_{\infty,\infty}\leq t)\geq 1-2dm\exp(-nt^2/(2\sigma^2\rho^+_{max}(1))).\nonumber
\end{align}
Let $\eta=2dm\exp(-nt^2/(2\sigma^2\rho^+_{max}(1)))$ and we can obtain \EqRef{eq:epsilonerror}.
\EqRef{eq:epsilonHerror} directly follows from \LemmaRef{lemma:epsilonHerror} and the following fact:
\begin{align}
\|\mathbf{x}_i\|^2\leq ay_i\Rightarrow\|X\|_F^2=\sum_{i=1}^m\|\mathbf{x}_i\|^2\leq ma\max_{i\in\mathbb{N}_m}{y_i}.\nonumber
\end{align}
\end{proof}

\section*{A.2 Proof of \LemmaRef{lemma:wgbound}}
\begin{proof}
The optimality condition of \EqRef{eq:relaxedlasso} implies that
\begin{align}
\frac{2}{n}X_i^T(X_i\hat{\mathbf{w}}_i-\mathbf{y}_i)+\hat{\bm{\lambda}}_i\odot\sign(\hat{\mathbf{w}}_i)=\mathbf{0},\nonumber
\end{align}
where $\odot$ denotes the element-wise product; $\sign(\mathbf{w})=[\sign(w_1),\cdots,\sign(w_d)]^T$, where $\sign(w_i)=1$, if $w_i>0$; $\sign(w_i)=-1$, if $w_i<0$;
and $\sign(w_i)\in[-1,1]$, otherwise.
We note that $X_i\hat{\mathbf{w}}_i-\mathbf{y}_i=X_i\hat{\mathbf{w}}_i-X_i\bar{\mathbf{w}}_i+X_i\bar{\mathbf{w}}_i-\mathbf{y}_i$ and we can rewrite the above equation into the following form:
\begin{align}
2A_i\Delta\hat{\mathbf{w}}_i=-2\bar{\bm{\epsilon}}_i-\hat{\bm{\lambda}}_i\odot\sign(\hat{\mathbf{w}}_i).\nonumber
\end{align}
Thus, for all $\mathbf{v}\in\mathbb{R}^{d}$, we have
\begin{align}\label{eq:innerprodeq}
2\mathbf{v}^TA_i\Delta\hat{\mathbf{w}}_i=-2\mathbf{v}^T\bar{\bm{\epsilon}}_i-\sum_{j=1}^d\hat{\lambda}_{ji}v_j\sign(\hat{w}_{ji}).
\end{align}
Letting $\mathbf{v}=\Delta\hat{\mathbf{w}}_i$ and noticing that $\Delta\hat{w}_{ji}=\hat{w}_{ji}$ for $(j,i)\notin\bar{\mathcal{F}}_i,i\in\mathbb{N}_m$, we obtain
\begin{align}
0&\leq 2\Delta\hat{\mathbf{w}}_i^TA_i\Delta\hat{\mathbf{w}}_i=-2\Delta\hat{\mathbf{w}}_i^T\bar{\bm{\epsilon}}_i-\sum_{j=1}^d\hat\lambda_{ji}\Delta\hat{w}_{ji}\sign(\hat{w}_{ji})\nonumber\\
&\leq 2\|\Delta\hat{\mathbf{w}}_i\|_1\|\bar{\bm{\epsilon}}_i\|_\infty-\sum_{(j,i)\in\bar{\mathcal{F}}_i}\hat\lambda_{ji}\Delta\hat{w}_{ji}\sign(\hat{w}_{ji})-\sum_{(j,i)\notin\bar{\mathcal{F}}_i}\hat\lambda_{ji}\Delta\hat{w}_{ji}\sign(\hat{w}_{ji})\nonumber\\
&\leq 2\|\Delta\hat{\mathbf{w}}_i\|_1\|\bar{\bm{\epsilon}}_i\|_\infty+\sum_{(j,i)\in\bar{\mathcal{F}}_i}\hat\lambda_{ji}|\Delta\hat{w}_{ji}|-\sum_{(j,i)\notin\bar{\mathcal{F}}_i}\hat\lambda_{ji}|\hat{w}_{ji}|\nonumber\\
&\leq 2\|\Delta\hat{\mathbf{w}}_i\|_1\|\bar{\bm{\epsilon}}_i\|_\infty+\sum_{(j,i)\in\bar{\mathcal{F}}_i}\hat\lambda_{ji}|\Delta\hat{w}_{ji}|-\sum_{(j,i)\in\mathcal{G}_i}\hat\lambda_{ji}|\hat{w}_{ji}|\nonumber\\
&\leq 2\|\Delta\hat{\mathbf{w}}_i\|_1\|\bar{\bm{\epsilon}}_i\|_\infty+\sum_{(j,i)\in\bar{\mathcal{F}}_i}\hat\lambda_{0i}|\Delta\hat{w}_{ji}|-\sum_{(j,i)\in\mathcal{G}_i}\hat\lambda_{\mathcal{G}_i}|\hat{w}_{ji}|\nonumber\\
&=\sum_{(j,i)\in\mathcal{G}_i}(2\|\bar{\bm{\epsilon}}_i\|_\infty-\hat\lambda_{\mathcal{G}_i})|\hat{w}_{ji}|+\sum_{(j,i)\notin\bar{\mathcal{F}_i}\cup\mathcal{G}_i}2\|\bar{\bm{\epsilon}}_i\|_\infty|\hat{w}_{ji}|+\sum_{(j,i)\in\bar{\mathcal{F}}_i}(2\|\bar{\bm{\epsilon}}_i\|_\infty+\hat\lambda_{0i})|\Delta\hat{w}_{ji}|.\nonumber
\end{align}
The last equality above is due to $\mathbb{N}_d\times\{i\}=\mathcal{G}_i\cup(\bar{\mathcal{F}}_i\cup\mathcal{G}_i)^c\cup\bar{\mathcal{F}}_i$ and $\Delta\hat{w}_{ji}=\hat{w}_{ji},\forall (j,i)\notin\bar{\mathcal{F}}_i\supseteq\mathcal{G}_i$.
Rearranging the above inequality and noticing that $2\|\bar{\bm{\epsilon}}_i\|_\infty<\hat\lambda_{\mathcal{G}_i}\leq\hat\lambda_{0i}$, we obtain:
\begin{align}\label{eq:wgboundvec}
\sum_{(j,i)\in\mathcal{G}_i}|\hat{w}_{ji}|&\leq\frac{2\|\bar{\bm{\epsilon}}_i\|_\infty}{\hat\lambda_{\mathcal{G}_i}-2\|\bar{\bm{\epsilon}}_i\|_\infty}\sum_{(j,i)\notin\bar{\mathcal{F}}_i\cup\mathcal{G}_i}|\hat{w}_{ji}|+\frac{2\|\bar{\bm{\epsilon}}_i\|_\infty+\hat\lambda_{0i}}{\hat\lambda_{\mathcal{G}_i}-2\|\bar{\bm{\epsilon}}_i\|_\infty}\sum_{(j,i)\in\bar{\mathcal{F}}_i}|\Delta\hat{w}_{ji}|\nonumber\\
&\leq\frac{2\|\bar{\bm{\epsilon}}_i\|_\infty+\hat\lambda_{0i}}{\hat\lambda_{\mathcal{G}_i}-2\|\bar{\bm{\epsilon}}_i\|_\infty}\|\Delta\hat{\mathbf{w}}_{\mathcal{G}_i^c}\|_1.
\end{align}
Then \LemmaRef{lemma:wgbound} can be obtained from the above inequality and the following two inequalities.
\begin{align}
%&\max_{i\in\mathbb{N}_m}\frac{2\|\bar{\bm{\epsilon}}_i\|_\infty}{\hat\lambda_{\mathcal{G}_i}-2\|\bar{\bm{\epsilon}}_i\|_\infty}\leq\frac{2\|\bar{\Upsilon}\|_{\infty,\infty}}{\hat\lambda_{\mathcal{G}}-2\|\bar{\Upsilon}\|_{\infty,\infty}}\nonumber\\
\max_{i\in\mathbb{N}_m}\frac{2\|\bar{\bm{\epsilon}}_i\|_\infty+\hat\lambda_{0i}}{\hat\lambda_{\mathcal{G}_i}-2\|\bar{\bm{\epsilon}}_i\|_\infty}\leq\frac{2\|\bar{\Upsilon}\|_{\infty,\infty}+\hat\lambda_{0}}{\hat\lambda_{\mathcal{G}}-2\|\bar{\Upsilon}\|_{\infty,\infty}}~\mathrm{and}~\sum_{i=1}^mx_iy_i\leq\|\mathbf{x}\|_\infty\|\mathbf{y}\|_1.\nonumber
\end{align}
\end{proof}

\section*{A.3 Proof of \LemmaRef{lemma:deltawbound}}
\begin{proof}
According to the definition of $\mathcal{G}~(\mathcal{G}_{(\ell)})$, we know that $\bar{\mathcal{F}}_i\cap\mathcal{G}_i=\emptyset~(i\in\mathbb{N}_m)$ and $\forall (j,i)\in\mathcal{G}~(\mathcal{G}_{(\ell)}),\hat{\lambda}_{ji}^{(\ell-1)}=\lambda$. Thus, all conditions of \LemmaRef{lemma:wgbound} are satisfied, by noticing
the relationship between \EqRef{eq:lambdacondition} and \EqRef{eq:epsilonerror}. Based on the definition of $\mathcal{G}~(\mathcal{G}_{(\ell)})$, we easily obtain $\forall j\in\mathbb{N}_d$:
\begin{align}\label{eq:samerow}
(j,i)\in\mathcal{G}_i,\forall i\in\mathbb{N}_m~\mathrm{or}~(j,i)\notin\mathcal{G}_i,\forall i\in\mathbb{N}_m.
\end{align}
and hence $k_{\ell}=|\mathcal{G}_1^c|=\cdots=|\mathcal{G}_m^c|$ ($k_{\ell}$ is some integer).
Now, we assume that at stage $\ell\geq 1$:
\begin{align}\label{eq:kleq2r}
k_{\ell}=|\mathcal{G}_1^c|=\cdots=|\mathcal{G}_m^c|\leq2\bar{r}.
\end{align}
We will show in the second part of this proof that \EqRef{eq:kleq2r} holds for all $\ell$. Based on \LemmaRef{lemma:piineq} and \EqRef{eq:eigenvalueineq}, we have:
\begin{align}
\pi_i\left(2\bar{r}+s,s\right)&\leq\frac{s^{1/2}}{2}\sqrt{\rho^+_i(s)/\rho^-_i(2\bar{r}+2s)-1}\nonumber\\
&\leq \frac{s^{1/2}}{2}\sqrt{1+s/(2\bar{r})-1}\nonumber\\
&=0.5s(2\bar{r})^{-1/2},\nonumber
\end{align}
which indicates that
\begin{align}
0.5\leq t_i=1-\pi_i(2\bar{r}+s,s)(2\bar{r})^{1/2}s^{-1}\leq 1.\nonumber
\end{align}
For all $t_i\in[0.5,1]$, under the conditions of \EqRef{eq:lambdacondition} and \EqRef{eq:epsilonerror}, we have
\begin{align}
\frac{2\|\bar{\bm{\epsilon}}_i\|_\infty+\lambda}{\lambda-2\|\bar{\bm{\epsilon}}_i\|_\infty}\leq\frac{2\|\bar{\Upsilon}\|_{\infty,\infty}+\lambda}{\lambda-2\|\bar{\Upsilon}\|_{\infty,\infty}}\leq\frac{7}{5}\leq\frac{4-t_i}{4-3t_i}\leq3.\nonumber
\end{align}
Following \LemmaRef{lemma:wgbound}, we have
\begin{align}
&\|\hat{W}_\mathcal{G}\|_{1,1}\leq 3\|\Delta\hat{W}_{\mathcal{G}^c}\|_{1,1}=3\|\Delta\hat{W}-\Delta\hat{W}_{\mathcal{G}}\|_{1,1}=3\|\Delta\hat{W}-\hat{W}_{\mathcal{G}}\|_{1,1}.\nonumber
\end{align}
Therefore
\begin{align}
&\|\Delta\hat{W}-\Delta\hat{W}_\mathcal{I}\|_{\infty,1}=\|\Delta\hat{W}_\mathcal{G}-\Delta\hat{W}_\mathcal{J}\|_{\infty,1}\nonumber\\
&\leq \|\Delta\hat{W}_\mathcal{J}\|_{1,1}/s=(\|\Delta\hat{W}_\mathcal{G}\|_{1,1}-\|\Delta\hat{W}-\Delta\hat{W}_\mathcal{I}\|_{1,1})/s\nonumber\\
&\leq s^{-1}(3\|\Delta\hat{W}-\hat{W}_{\mathcal{G}}\|_{1,1}-\|\Delta\hat{W}-\Delta\hat{W}_\mathcal{I}\|_{1,1}),\nonumber
\end{align}
which implies that
\begin{align}
&\|\Delta\hat{W}\|_{2,1}-\|\Delta\hat{W}_\mathcal{I}\|_{2,1}\leq\|\Delta\hat{W}-\Delta\hat{W}_\mathcal{I}\|_{2,1}\nonumber\\
&\leq(\|\Delta\hat{W}-\Delta\hat{W}_\mathcal{I}\|_{1,1}\|\Delta\hat{W}-\Delta\hat{W}_\mathcal{I}\|_{\infty,1})^{1/2}\nonumber\\
&\leq\left(\|\Delta\hat{W}-\Delta\hat{W}_\mathcal{I}\|_{1,1}\right)^{1/2}\left(s^{-1}(3\|\Delta\hat{W}-\hat{W}_{\mathcal{G}}\|_{1,1}-\|\Delta\hat{W}-\Delta\hat{W}_\mathcal{I}\|_{1,1})\right)^{1/2}\nonumber\\
&\leq \left(\left(3\|\Delta\hat{W}-\hat{W}_{\mathcal{G}}\|_{1,1}/2\right)^2\right)^{1/2}s^{-1/2}\nonumber\\
&\leq (3/2)s^{-1/2}(2\bar{r})^{1/2}\|\Delta\hat{W}-\hat{W}_{\mathcal{G}}\|_{2,1}\nonumber\\
&\leq (3/2)(2\bar{r}/s)^{1/2}\|\Delta\hat{W}_{\mathcal{I}}\|_{2,1}.\nonumber
\end{align}
In the above derivation, the third inequality is due to $a(3b-a)\leq(3b/2)^2$, and the fourth inequality follows from \EqRef{eq:kleq2r} and $\bar{\mathcal{F}}\cap\mathcal{G}=\emptyset\Rightarrow\Delta\hat{W}_{\mathcal{G}}=\hat{W}_{\mathcal{G}}$.
Rearranging the above inequality, we obtain at stage $\ell$:
\begin{align}
&\|\Delta\hat{W}\|_{2,1}\leq\left(1+1.5\sqrt{\frac{2\bar{r}}{s}}\right)\|\Delta\hat{W}_{\mathcal{I}}\|_{2,1}.\label{eq:deltawboundtemp}
\end{align}
From \LemmaRef{lemma:pirhoineq}, we have:
\begin{align}
&\max(0,\Delta\hat{\mathbf{w}}_{\mathcal{I}_i}^TA_i\Delta\hat{\mathbf{w}}_i)\nonumber\\
&\geq\rho^-_i(k_{\ell}+s)(\|\Delta\hat{\mathbf{w}}_{\mathcal{I}_i}\|-\pi_i(k_{\ell}+s,s)\|\hat{\mathbf{w}}_{\mathcal{G}_i}\|_1/s)\|\Delta\hat{\mathbf{w}}_{\mathcal{I}_i}\|\nonumber\\
&\geq\rho^-_i(k_{\ell}+s)[1-(1-t_i)(4-t_i)/(4-3t_i)]\|\Delta\hat{\mathbf{w}}_{\mathcal{I}_i}\|^2\nonumber\\
&\geq0.5t_i\rho^-_i(k_{\ell}+s)\|\Delta\hat{\mathbf{w}}_{\mathcal{I}_i}\|^2\nonumber\\
&\geq0.25\rho^-_i(2\bar{r}+s)\|\Delta\hat{\mathbf{w}}_{\mathcal{I}_i}\|^2\nonumber\\
&\geq0.25\rho^-_{min}(2\bar{r}+s)\|\Delta\hat{\mathbf{w}}_{\mathcal{I}_i}\|^2,\nonumber
\end{align}
where the second inequality is due to \EqRef{eq:wgboundvec}, that is
\begin{align}
\|\hat{\mathbf{w}}_{\mathcal{G}_i}\|_1&\leq\frac{2\|\bar{\bm{\epsilon}}_i\|_\infty+\hat\lambda_{0i}}{\hat\lambda_{\mathcal{G}_i}-2\|\bar{\bm{\epsilon}}_i\|_\infty}\|\Delta\hat{\mathbf{w}}_{\mathcal{G}_i^c}\|_1\nonumber\\
&\leq\frac{(2\|\bar{\bm{\epsilon}}_i\|_\infty+\hat\lambda_{0i})\sqrt{k_{\ell}}}{\hat\lambda_{\mathcal{G}_i}-2\|\bar{\bm{\epsilon}}_i\|_\infty}\|\Delta\hat{\mathbf{w}}_{\mathcal{G}_i^c}\|\nonumber\\
&\leq \frac{(2\|\bar{\bm{\epsilon}}_i\|_\infty+\hat\lambda_{0i})\sqrt{k_{\ell}}}{\hat\lambda_{\mathcal{G}_i}-2\|\bar{\bm{\epsilon}}_i\|_\infty}\|\Delta\hat{\mathbf{w}}_{\mathcal{I}_i}\|\nonumber\\
&\leq \frac{(4-t_i)\sqrt{k_{\ell}}}{4-3t_i}\|\Delta\hat{\mathbf{w}}_{\mathcal{I}_i}\|;\nonumber
\end{align}
the third inequality follows from $1-(1-t_i)(4-t_i)/(4-3t_i)\geq 0.5t_i$ for $t_i\in [0.5,1]$ and the fourth inequality follows from the assumption in \EqRef{eq:kleq2r} and $t_i\geq0.5$.

If $\Delta\hat{\mathbf{w}}_{\mathcal{I}_i}^TA_i\Delta\hat{\mathbf{w}}_i\leq 0$, then $\|\Delta\hat{\mathbf{w}}_{\mathcal{I}_i}\|=0$. If $\Delta\hat{\mathbf{w}}_{\mathcal{I}_i}^TA_i\Delta\hat{\mathbf{w}}_i>0$, then we have
\begin{align}
\Delta\hat{\mathbf{w}}_{\mathcal{I}_i}^TA_i\Delta\hat{\mathbf{w}}_i\geq0.25\rho^-_{min}(2\bar{r}+s)\|\Delta\hat{\mathbf{w}}_{\mathcal{I}_i}\|^2.\label{eq:wi1}
\end{align}
By letting $\mathbf{v}=\Delta\hat{\mathbf{w}}_{\mathcal{I}_i}$, we obtain the following from \EqRef{eq:innerprodeq}:
\begin{align}
&2\Delta\hat{\mathbf{w}}_{\mathcal{I}_i}^TA_i\Delta\hat{\mathbf{w}}_i=-2\Delta\hat{\mathbf{w}}_{\mathcal{I}_i}^T\bar{\bm{\epsilon}}_i-\sum_{(j,i)\in\mathcal{I}_i}\hat\lambda_{ji}\Delta\hat{w}_{ji}\sign(\hat{w}_{ji})\nonumber\\
&=-2\Delta\hat{\mathbf{w}}_{\mathcal{I}_i}^T\bar{\bm{\epsilon}}_{\mathcal{G}_i^c}-2\Delta\hat{\mathbf{w}}_{\mathcal{I}_i}^T\bar{\bm{\epsilon}}_{\mathcal{G}_i}-\sum_{(j,i)\in\bar{\mathcal{F}}_i}\hat\lambda_{ji}\Delta\hat{w}_{ji}\sign(\hat{w}_{ji})-\sum_{(j,i)\in\mathcal{J}_i}\hat\lambda_{ji}|\Delta\hat{w}_{ji}|-\sum_{(j,i)\in\bar{\mathcal{F}}_i^c\cap\mathcal{G}_i^c}\hat\lambda_{ji}|\Delta\hat{w}_{ji}|\nonumber\\
&=-2\Delta\hat{\mathbf{w}}_{\mathcal{I}_i}^T\bar{\bm{\epsilon}}_{\mathcal{G}_i^c}-2\Delta\hat{\mathbf{w}}_{\mathcal{J}_i}^T\bar{\bm{\epsilon}}_{\mathcal{J}_i}-\sum_{(j,i)\in\bar{\mathcal{F}}_i}\hat\lambda_{ji}\Delta\hat{w}_{ji}\sign(\hat{w}_{ji})-\sum_{(j,i)\in\mathcal{J}_i}\hat\lambda_{ji}|\Delta\hat{w}_{ji}|-\sum_{(j,i)\in\bar{\mathcal{F}}_i^c\cap\mathcal{G}_i^c}\hat\lambda_{ji}|\Delta\hat{w}_{ji}|\nonumber\\
&\leq 2\|\Delta\hat{\mathbf{w}}_{\mathcal{I}_i}\|\|\bar{\bm{\epsilon}}_{\mathcal{G}_i^c}\|+2\|\bar{\bm{\epsilon}}_{\mathcal{J}_i}\|_\infty\sum_{(j,i)\in\mathcal{J}_i}|\Delta\hat{w}_{ji}|+\sum_{(j,i)\in\bar{\mathcal{F}}_i}\hat\lambda_{ji}|\Delta\hat{w}_{ji}|-\sum_{(j,i)\in\mathcal{J}_i}\hat\lambda_{ji}|\Delta\hat{w}_{ji}|\nonumber\\
&\leq 2\|\Delta\hat{\mathbf{w}}_{\mathcal{I}_i}\|\|\bar{\bm{\epsilon}}_{\mathcal{G}_i^c}\|+\left(\sum_{(j,i)\in\bar{\mathcal{F}}_i}\hat\lambda_{ji}^2\right)^{1/2}\|\Delta\hat{\mathbf{w}}_{\bar{\mathcal{F}}_i}\|\nonumber\\
&\leq 2\|\Delta\hat{\mathbf{w}}_{\mathcal{I}_i}\|\|\bar{\bm{\epsilon}}_{\mathcal{G}_i^c}\|+\left(\sum_{(j,i)\in\bar{\mathcal{F}}_i}\hat\lambda_{ji}^2\right)^{1/2}\|\Delta\hat{\mathbf{w}}_{\mathcal{I}_i}\|.\label{eq:wi2}
\end{align}
In the above derivation, the second equality is due to $\mathcal{I}_i=\mathcal{J}_i\cup\bar{\mathcal{F}}_i\cup(\bar{\mathcal{F}}_i^c\cap\mathcal{G}_i^c)$; the third equality is due to $\mathcal{I}_i\cap\mathcal{G}_i=\mathcal{J}_i$; the second inequality follows from $\forall (j,i)\in\mathcal{J}_i,\hat\lambda_{ji}=\lambda\geq2\|\bar{\bm{\epsilon}}_i\|_\infty\geq2\|\bar{\bm{\epsilon}}_{\mathcal{J}_i}\|_\infty$ and the last inequality follows from $\bar{\mathcal{F}}_i\subseteq\mathcal{G}_i^c\subseteq\mathcal{I}_i$.
Combining \EqRef{eq:wi1} and \EqRef{eq:wi2}, we have
\begin{align}
\|\Delta\hat{\mathbf{w}}_{\mathcal{I}_i}\|\leq\frac{2}{\rho^-_{min}(2\bar{r}+s)}\left[2\|\bar{\bm{\epsilon}}_{\mathcal{G}_i^c}\|+\left(\sum_{(j,i)\in\bar{\mathcal{F}}_i}\hat\lambda_{ji}^2\right)^{1/2}\right].\nonumber
\end{align}
Notice that
\begin{align}
&\|\mathbf{x}_i\|\leq a(\|\mathbf{y}_i\|+\|\mathbf{z}_i\|)\Rightarrow\|X\|_{2,1}^2\leq m\|X\|_F^2=m\sum_{i}\|\mathbf{x}_i\|^2\leq 2ma^2(\|Y\|_F^2+\|Z\|_F^2).\nonumber
\end{align}
Thus, we have
\begin{align}\label{eq:wi21error}
\|\Delta\hat{W}_{\mathcal{I}}\|_{2,1}\leq\frac{\sqrt{8m\left(4\|\bar{\Upsilon}_{\mathcal{G}_{(\ell)}^c}\|_F^2+\sum_{(j,i)\in\bar{\mathcal{F}}}(\hat\lambda_{ji}^{(\ell-1)})^2\right)}}{\rho^-_{min}(2\bar{r}+s)}.
\end{align}
Therefore, at stage $\ell$, \EqRef{eq:deltawbound} in \LemmaRef{lemma:deltawbound} directly follows from \EqRef{eq:deltawboundtemp} and \EqRef{eq:wi21error}.
Following \EqRef{eq:deltawbound}, we have:
\begin{align}
&\|\hat{W}^{(\ell)}-\bar{W}\|_{2,1}=\|\Delta\hat{W}^{(\ell)}\|_{2,1}\nonumber\\
&\leq\frac{\left(1+1.5\sqrt{\frac{2\bar{r}}{s}}\right)\sqrt{8m\left(4\|\bar{\Upsilon}_{\mathcal{G}_{(\ell)}^c}\|_F^2+\sum_{(j,i)\in\bar{\mathcal{F}}}(\hat\lambda_{ji}^{(\ell-1)})^2\right)}}{\rho^-_{min}(2\bar{r}+s)}\nonumber\\
&\leq \frac{8.83\sqrt{m}\sqrt{4\|\Upsilon\|^2_{\infty,\infty}|\mathcal{G}_{(\ell)}^c|+\bar{r}m\lambda^2}}{\rho^-_{min}(2\bar{r}+s)}\nonumber\\
&\leq \frac{8.83\sqrt{m}\lambda\sqrt{\frac{8}{144}\bar{r}m+\bar{r}m}}{\rho^-_{min}(2\bar{r}+s)}\nonumber\\
&\leq\frac{9.1m\lambda\sqrt{\bar{r}}}{\rho^-_{min}(2\bar{r}+s)},\nonumber
\end{align}
where the first inequality is due to \EqRef{eq:wi21error}; the second inequality is due
to $s\geq\bar{r}$~(assumption~in~\ThemRef{theorem:mainbound}), $\hat{\lambda}_{ji}\leq\lambda,~\bar{r}m=|\bar{\mathcal{H}}|\geq|\bar{\mathcal{F}}|$
and the third inequality follows from \EqRef{eq:kleq2r} and $\|\bar{\Upsilon}\|^2_{\infty,\infty}\leq(1/144)\lambda^2$.
Therefore, \EqRef{eq:w21error} in \LemmaRef{lemma:deltawbound} holds at stage $\ell$.

Notice that we obtain \LemmaRef{lemma:deltawbound} at stage $\ell$, by assuming that \EqRef{eq:kleq2r} is satisfied.
To prove that \LemmaRef{lemma:deltawbound} holds for all stages, we next need to
prove by induction that \EqRef{eq:kleq2r} holds at all stages.

When $\ell=1$, we have $\mathcal{G}_{(1)}^c=\bar{\mathcal{H}}$, which implies that \EqRef{eq:kleq2r} holds. Now, we assume
that \EqRef{eq:kleq2r} holds at stage $\ell$. Thus, by hypothesis induction, we have:
%\begin{align}
%&\|\hat{W}^{(\ell)}-\bar{W}\|_{2,1}=\|\Delta\hat{W}^{(\ell)}\|_{2,1}\nonumber\\
%&\leq\frac{\left(1+1.5\sqrt{\frac{2\bar{r}}{s}}\right)\sqrt{8m\left(4\|\bar{\Upsilon}_{\mathcal{G}_{(\ell)}^c}\|_F^2+\sum_{(j,i)\in\bar{\mathcal{F}}}(\hat\lambda_{ji}^{(\ell-1)})^2\right)}}{\rho^-_{min}(2\bar{r}+s)}\nonumber\\
%&\leq \frac{8.83\sqrt{m}\sqrt{4\|\Upsilon\|^2_{\infty,\infty}|\mathcal{G}_{(\ell)}^c|+\bar{r}m\lambda^2}}{\rho^-_{min}(2\bar{r}+s)}\nonumber\\
%&\leq \frac{8.83\sqrt{m}\lambda\sqrt{\frac{8}{144}\bar{r}m+\bar{r}m}}{\rho^-_{min}(2\bar{r}+s)}\leq\frac{9.1m\lambda\sqrt{\bar{r}}}{\rho^-_{min}(2\bar{r}+s)},\nonumber
%\end{align}
%where the first inequality is due to the hypothesis induction; the second inequality is due
%to $s\geq\bar{r}~(\mathrm{assumption~in~\ThemRef{theorem:mainbound}}),~\hat{\lambda}_{ji}\leq\lambda,~\bar{r}m=|\bar{\mathcal{H}}|\geq|\bar{\mathcal{F}}|$
%and the third inequality follows from \EqRef{eq:kleq2r} and $\|\bar{\Upsilon}\|^2_{\infty,\infty}\leq(1/144)\lambda^2$.
%Thus, we obtain:
\begin{align}
\sqrt{|\mathcal{G}_{(\ell+1)}^c\setminus\bar{\mathcal{H}}|}&\leq\sqrt{m\theta^{-2}\|\hat{W}^{(\ell)}_{\mathcal{G}_{(\ell+1)}^c\setminus\bar{\mathcal{H}}}-\bar{W}_{\mathcal{G}_{(\ell+1)}^c\setminus\bar{\mathcal{H}}}\|_{2,1}^2}\nonumber\\
&\leq\sqrt{m}\theta^{-1}\left\|\hat{W}^{(\ell)}-\bar{W}\right\|_{2,1}\nonumber\\
&\leq\frac{9.1m^{3/2}\lambda\sqrt{\bar{r}}\theta^{-1}}{\rho^-_{min}(2\bar{r}+s)}\nonumber\\
&\leq\sqrt{\bar{r}m},\nonumber
\end{align}
where $\theta$ is the thresholding parameter in \EqRef{eq:msmtfl}; the first inequality above follows from the definition of $\mathcal{G}_{(\ell)}$ in \LemmaRef{lemma:deltawbound}:
\begin{align}
\forall &(j,i)\in\mathcal{G}_{(\ell+1)}^c\setminus\bar{\mathcal{H}},\|(\hat{\mathbf{w}}^{(\ell)})^j\|_1^2/\theta^2=\|(\hat{\mathbf{w}}^{(\ell)})^j-\bar{\mathbf{w}}^j\|_1^2/\theta^2\geq1\nonumber\\
\Rightarrow&|\mathcal{G}_{(\ell+1)}^c\setminus\bar{\mathcal{H}}|\leq m\theta^{-2}\|\hat{W}^{(\ell)}_{\mathcal{G}_{(\ell+1)}^c\setminus\bar{\mathcal{H}}}-\bar{W}_{\mathcal{G}_{(\ell+1)}^c\setminus\bar{\mathcal{H}}}\|_{2,1}^2;\nonumber
\end{align}
the last inequality is due to \EqRef{eq:thetacondition}.
Thus, we have:
\begin{align}
|\mathcal{G}_{(\ell+1)}^c\setminus\bar{\mathcal{H}}|\leq\bar{r}m\Rightarrow|\mathcal{G}_{(\ell+1)}^c|\leq2\bar{r}m\Rightarrow k_{\ell+1}\leq2\bar{r}.\nonumber
\end{align}
Therefore, \EqRef{eq:kleq2r} holds at all stages. Thus the two inequalities in \LemmaRef{lemma:deltawbound} hold at all stages. This completes the proof of the lemma.
\end{proof}

\section*{A.4 Proof of \LemmaRef{lemma:lambdadecomp}}
\begin{proof}
The first inequality directly follows from $\bar{\mathcal{H}}\supseteq\bar{\mathcal{F}}$. Next, we focus on the second inequality. For each $(j,i)\in\bar{\mathcal{F}}~(\bar{\mathcal{H}})$, if $\|\hat{\mathbf{w}}^j\|_1<\theta$, by considering \EqRef{eq:noiselevel}, we have
\begin{align}
\|\bar{\mathbf{w}}^j-\hat{\mathbf{w}}^j\|_1\geq\|\bar{\mathbf{w}}^j\|_1-\|\hat{\mathbf{w}}^j\|_1\geq2\theta-\theta=\theta.\nonumber
\end{align}
Therefore, we have for each $(j,i)\in\bar{\mathcal{F}}~(\bar{\mathcal{H}})$:
\begin{align}
I\left(\|\hat{\mathbf{w}}^j\|_1<\theta\right)\leq \|\bar{\mathbf{w}}^j-\hat{\mathbf{w}}^j\|_1/\theta.\nonumber
\end{align}
Thus, the second inequality of \LemmaRef{lemma:lambdadecomp} directly follows from the above inequality.
\end{proof}

\section*{B. Lemmas from \citet{zhang2010analysis}}
\begin{lemma}\label{lemma:subgaussianproperty}
Let $\mathbf{a}\in\mathbb{R}^n$ be a fixed vector and $\mathbf{x}\in\mathbb{R}^n$ be a random vector which
is composed of independent sub-Gaussian components with parameter $\sigma$. Then we have:
\begin{align}
\Pr(|\mathbf{a}^T\mathbf{x}|\geq t)\leq 2\exp\left(-t^2/(2\sigma^2\|\mathbf{a}\|^2)\right),\forall t>0.\nonumber
\end{align}
\end{lemma}
\begin{lemma}\label{lemma:piineq}%$\pi_i(k_i,s_i)\leq \frac{s_i^{1/2}}{2}\sqrt{\rho^+_i(s_i)/\rho^-_i(k_i+s_i)-1}$.
The following inequality holds:
\begin{align}
\pi_i(k_i,s_i)\leq \frac{s_i^{1/2}}{2}\sqrt{\rho^+_i(s_i)/\rho^-_i(k_i+s_i)-1}.\nonumber
\end{align}
\end{lemma}

\begin{lemma}\label{lemma:pirhoineq}
Let $\mathcal{G}_i\subseteq\mathbb{N}_{d}\times\{i\}$ such that $|\mathcal{G}_i^c|=k_i$, and let
$\mathcal{J}_i$ be indices of the $s_i$ largest components (in absolute values) of $\mathbf{w}_{\mathcal{G}_i}$ and
$\mathcal{I}_i=\mathcal{G}_i^c\cup\mathcal{J}_i$. Then for any $\mathbf{w}_i\in\mathbb{R}^{d}$, we have
\begin{align}
&\max(0,\mathbf{w}_{\mathcal{I}_i}^TA_i\mathbf{w}_i)\geq\rho^-_i(k_i+s_i)(\|\mathbf{w}_{\mathcal{I}_i}\|-\pi_i(k_i+s_i,s_i)\|\mathbf{w}_{\mathcal{G}_i}\|_1/s_i)\|\mathbf{w}_{\mathcal{I}_i}\|.\nonumber
\end{align}
\end{lemma}

\begin{lemma}\label{lemma:epsilonHerror}
Let $\bar{\bm{\epsilon}}_i=[\bar{\epsilon}_{1i},\cdots,\bar{\epsilon}_{di}]=\frac{1}{n}X_i^T(X_i\bar{\mathbf{w}}_i-\mathbf{y}_i)~(i\in\mathbb{N}_m)$,
and $\bar{\mathcal{H}}_i\subseteq\mathbb{N}_d\times\{i\}$. Under the conditions of \AssumeRef{assumption:subgaussian}, the followings
hold with probability larger than $1-\eta$:
\begin{align}
&\|\bar{\bm{\epsilon}}_{\bar{\mathcal{H}}_i}\|^2\leq\sigma^2\rho^+_i(|\bar{\mathcal{H}}_i|)(7.4|\bar{\mathcal{H}}_i|+2.7\ln(2/\eta))/n.\nonumber
\end{align}
\end{lemma}

\vskip 0.2in
\bibliography{jmlrgph}

\end{document}